\documentclass[twoside,11pt]{article}

\usepackage[abbrvbib, preprint]{jmlr2e}
\usepackage{xcolor}

\usepackage{jmlr2e}
\usepackage{amsmath}
\usepackage{algorithm}
\usepackage{algorithmic}
\usepackage{graphicx}
\usepackage{multirow}
\usepackage{subcaption}

\makeatletter
\newcommand{\rmnum}[1]{\romannumeral #1}
\newcommand{\Rmnum}[1]{\expandafter\@slowromancap\romannumeral #1@}
\makeatother

\newtheorem{claim}[theorem]{Claim}
\newtheorem{assumption}[theorem]{Assumption}

\ShortHeadings{Randomized Greedy Algorithms for $k$-Center Clustering with Outliers}{Ding, Huang, Liu, Yu, and Wang}
\firstpageno{1}

\begin{document}

\title{Randomized Greedy Algorithms and Composable Coreset for $k$-Center Clustering with Outliers\thanks{This work was supported in part by National Key R\&D program of China through grant 2021YFA1000900. A preliminary version of this paper has appeared in 27th Annual European Symposium on Algorithms (ESA2019)~\citep{DBLP:conf/esa/DingYW19}.}}

\author{\name Hu Ding \email huding@ustc.edu.cn \\
       \addr School of Computer Science and Technology\\
       University of Science and Technology of China\\
       Anhui, China
       \AND
       \name Ruomin Huang\email hrm@mail.ustc.edu.cn \\
       \addr School of Data Science\\
       University of Science and Technology of China\\
       Anhui, China
       \AND
       \name Kai Liu\email liukai0010@mail.ustc.edu.cn \\
       \addr School of Computer Science and Technology \\
       University of Science and Technology of China\\
       Anhui, China
       \AND
       \name Haikuo Yu\email yhk7786@mail.ustc.edu.cn \\
       \addr School of Computer Science and Technology \\
       University of Science and Technology of China\\
       Anhui, China
       \AND
       \name Zixiu Wang\email wzx2014@mail.ustc.edu.cn \\
       \addr School of Computer Science and Technology \\
       University of Science and Technology of China\\
       Anhui, China
       }

\editor{unknown}

\maketitle

\begin{abstract}
In this paper, we study the problem of {\em $k$-center clustering with outliers}.  The problem  has many important applications in real world, but the presence of outliers  can significantly increase the computational complexity.  Though a number of methods have been developed in the past decades, it is still quite challenging to design quality guaranteed algorithm with low complexity for this problem. 
Our idea is inspired by the greedy method, Gonzalez's algorithm, that was developed for solving the ordinary $k$-center clustering problem. Based on some novel observations, we show that a simple randomized version of this greedy strategy actually can handle outliers efficiently.  We further show that this randomized greedy approach also yields small coreset for the problem in doubling metrics (even if the doubling dimension is not given), which can greatly reduce the computational  complexity. Moreover, together with the partial clustering framework proposed by~\citet{guha2017distributed}, we prove that our coreset method can be applied to distributed data with a low communication complexity.   The experimental results suggest that our algorithms can achieve near optimal solutions and yield lower complexities comparing with the existing methods.
\end{abstract}


\section{Introduction}
\label{sec-intro}
{\em Clustering} is one of the most fundamental problems that has been widely applied in the fields of machine learning and data mining~\citep{jain2010data}. Given a set of elements, the goal of clustering is to partition the input set into several groups based on their similarities or dissimilarities. 
Several clustering models have been extensively studied, such as the $k$-center, $k$-median, and $k$-means clusterings~\citep{awasthi2014center}. 
In practice, the data sets often contain outliers. In particular, the outliers can be arbitrarily located in the space, e.g., an adversarial attacker can inject a small number of specially crafted samples into the  data~\citep{DBLP:journals/pr/BiggioR18}. Even a small number of outliers could seriously destroy the final clustering result~\citep{chandola2009anomaly}. The  clustering with outliers problem is also closely related to the topics like {\em robust statistics}~\citep{DBLP:journals/siamcomp/DiakonikolasKKL19} and {\em outliers removal}~\citep{schubert2017dbscan}. The key difference with these topics is that the focus of clustering with outliers is to optimize the clustering objective function via excluding a small number of outliers.

In this paper, we focus on the problem of {\em $k$-center clustering with outliers}. Given a metric space with $n$ vertices and a pre-specified number of outliers $z<n$, the problem is to find $k$ balls to cover at least $n-z$ vertices and minimize the maximum radius of the balls. The problem  can be also defined in Euclidean space so that the cluster centers can be any points in the space (i.e., not restricted to be selected from the input points). The $k$-center clustering with outliers problem can be viewed as a generalization of the ordinary $k$-center clustering problem (i.e., the number of outliers $z=0$).  The ordinary $k$-center clustering has many important applications in machine learning, such as deep learning~\citep{DBLP:conf/iclr/ColemanYMMBLLZ20}, active learning~\citep{DBLP:conf/iclr/SenerS18}, and fairness~\citep{DBLP:conf/icml/KleindessnerAM19}. The $2$-approximation algorithms for ordinary $k$-center clustering (without outliers) were given by \cite{gonzalez1985clustering} and~\cite{hochbaum1985best}, where the ``approximation ratio'' is the ratio of the obtained radius to the optimal one. It was also proved that any approximation ratio lower than ``$2$'' implies $P=NP$.

Comparing with the ordinary $k$-center clustering problem, the challenge for solving the case with outliers can be greatly increased. For example, there are ${n\choose z}$ different cases that need to consider for optimizing the objective if we do not know who are the outliers in advance. The number ${n\choose z}$ can be quite large even if $z$ is a constant number. So existing algorithms often suffer from the issue of high computational complexity.  
A $3$-approximation algorithm for $k$-center clustering with outliers in arbitrary metrics was proposed by \cite{charikar2001algorithms}. The time complexity of their algorithm is $O(k n^2\log n)$ (or $O(k n^2 D\log n)$ in a $D$-dimensional Euclidean space) which is quadratic in the input size $n$. A following streaming $(4+\epsilon)$-approximation algorithm was proposed by \cite{mccutchen2008streaming}. The time complexity is $O\big(\frac{1}{\epsilon}(kzn+(kz)^2\log \Phi)\big)$, where $\Phi$ is the ratio of the optimal radius to the smallest pairwise distance among the vertices (e.g., if $z=5\% n$, the complexity is quadratic in the input size $n$). \cite{DBLP:conf/esa/BergMZ21} proposed the first streaming algorithm in the sliding-window model based on the static approximation algorithm of~\cite{charikar2001algorithms}. 
Recently, \cite{DBLP:conf/icalp/ChakrabartyGK16} proposed a $2$-approximation algorithm for metric $k$-center clustering with outliers, but 
the algorithm   needs to solve a complicated model of linear programming and the exact time complexity is not provided. 

Obviously, when the input data size is large, these existing algorithms cannot be efficiently implemented in practice. Therefore, from both the theoretical and practical perspectives, an interesting question is that whether we can reduce the computational complexity of $k$-center clustering with outliers with preserving the clustering quality guarantee.

\subsection{Our Contributions}
In this paper, our contributions are threefold. 

\vspace{0.05in}
\textbf{(1)} First, we show that a simple randomized greedy data selection strategy can yield a quality guaranteed solution with linear time complexity (Section~\ref{sec-center}). 
Our idea is inspired by the greedy method from \cite{gonzalez1985clustering} which was developed for solving the ordinary $k$-center clustering. The Gonzalez's algorithm greedily selects $k$ points iteratively, where each iteration takes the point that has the largest distance to the set of already selected points.  Based on some novel insights, we show that a randomized version of this greedy method also works for the problem with outliers. Roughly speaking, we replace each greedy selection step by a bi-level ``\textbf{greedy selection$+$random sampling}'' step: select the farthest $(1+\epsilon)z$ points (rather than the farthest single point) with a small parameter $\epsilon\in (0,1)$, and then take a random sample from this selected set. 
Our approach can achieve the approximation ratio ``$2$'' with respect to the clustering cost (i.e., the radius), if $(1+\epsilon)z$ (slightly more than the pre-specified number $z$) outliers are allowed to be discarded; moreover, the time complexity is linear in the input size. Another advantage of our method is that it can be further improved to be sublinear time, that is, the time complexity can be independent of the input data size $n$. Thus our result is a significantly improvement upon the previous approximation  algorithms 
on time complexity.

Being independent of our preliminary work~\citep{DBLP:conf/esa/DingYW19}, \cite{NEURIPS2019_73983c01} proposed a similar greedy algorithm for $k$-center clustering with outliers, but their clustering approximation ratio is $2+\delta$ ($\delta\in(0,1)$). Also, it is unclear that whether their runtime can be improved to be sublinear. 

\vspace{0.05in}

\textbf{(2)} 
We then study the coreset construction problem for $k$-center clustering with outliers. Given a large data set $X$, the technique of ``coreset'' is to generate a much smaller set  $S$ that can approximately preserve the structure of $X$; therefore we can run any existing algorithm on $S$ so as to reduce the total  complexity~\citep{DBLP:journals/widm/Feldman20}. 

We consider the uniform sampling approach for coreset construction first. \cite{charikar2003better} showed that the uniform random sampling technique can be applied to reduce the data size for metric $k$-center clustering with outliers. Recently, \cite{huang2018epsilon} showed a similar result for the problem in Euclidean space. In Section~\ref{sec-core1}, we revisit the result of  \cite{huang2018epsilon} and provide a more careful analysis. In particular, we show that the sample size can be reduced by a factor of $\frac{1}{\gamma}$ where $\gamma=\frac{z}{n}$. This improvement could be important for the case $z\ll n$, e.g., $z=\sqrt{n}$.

Although the uniform sampling approach is very easy to implement, it is not a standard coreset since it always incurs an inevitable error on the number of discarded outliers. So we further consider to build a coreset that can remedy this issue, but we need to add some mild assumption first. 
Many real-world data sets have low intrinsic  dimensions~\citep{belkin2003problems}. 
For example, image sets usually can be represented in low dimensional manifold though the Euclidean dimension of the image vectors can be very high. The ``doubling dimension'' is widely used for measuring the intrinsic dimensions of data sets~\citep{talwar2004bypassing} (the formal definition is given in Section~\ref{sec-pre}).  With the ``low doubling dimension'' assumption, we show that our aforementioned randomized greedy approach can be used to construct a coreset that incurs no error on the number of outliers (Section~\ref{sec-doubling}). 
The size of our coreset is $2z+\tilde O\big((\frac{2}{\mu})^\rho k\big)$, where $\rho$ is the doubling dimension and $\mu\in (0,1)$ is the small parameter measuring the quality of the coreset; the construction time is $\tilde O((\frac{2}{\mu})^\rho kn)$. 
Recently, \cite{DBLP:journals/corr/abs-1802-09205} also provided a coreset for $k$-center clustering with $z$ outliers in doubling metrics, where their coreset size is $T=O((k+z)(\frac{24}{\mu})^\rho)$ with $O(nT )$ construction time. 
So our result is a significant improvement upon their result in terms of both coreset size and construction time. Please see Table \ref{tab:coreset}  for details. Comparing with the results of \cite{DBLP:journals/corr/abs-1802-09205}, another  advantage of our approach is that we only assume that \textbf{the inliers of the given data have a low doubling dimension $\rho>0$.} We do not have any assumption on the outliers; namely, the outliers can scatter arbitrarily in the space (e.g., the outliers may be added by an adversarial attacker~\citep{DBLP:journals/pr/BiggioR18}). We believe that this assumption captures a large range of high dimensional instances in practice.

\begin{table}[]
\resizebox{\textwidth}{!}{%
\begin{tabular}{cc|c|c}
\hline
\multicolumn{2}{c|}{\textbf{Methods}} &
  \textbf{Size} &
  \textbf{Construction Time} \\ \hline
\multicolumn{1}{c|}{\multirow{2}{*}{Uniform sampling}} &
  \cite{huang2018epsilon} &
  $\tilde{O}\left(\frac{n^2}{\epsilon^2 z^2} k D\right)$ &
    \\ \cline{2-4} 
\multicolumn{1}{c|}{} &
  This paper (Theorem~\ref{the-samplereduce}) &
  $\tilde{O}\left(\frac{n}{\epsilon^2 z} k D\right)$ &
   \\ \hline
\multicolumn{1}{c|}{\multirow{2}{*}{$\mu$-Coreset}} &
  \cite{DBLP:journals/corr/abs-1802-09205} &
  $O\left((k+z)\left(\frac{24}{\mu}\right)^\rho\right)$ &
  $O\left((k+z)\left(\frac{24}{\mu}\right)^\rho n\right)$ \\ \cline{2-4} 
\multicolumn{1}{c|}{} &
  This paper (Theorem~\ref{the-coreset}) &
  $2 z+\tilde O\left(\left(\frac{2}{\mu}\right)^\rho k\right)$ &
  $\tilde O\left(\left(\frac{2}{\mu}\right)^\rho kn\right)$ \\ \hline
\end{tabular}%
}
\captionsetup{format=hang}
\caption{Existing and our data compressing method for $k$-center clustering with $z$ outliers. ``$D$'' and ``$\rho$'' stand for the dimension of the Euclidean space and doubling dimension, respectively.}
\label{tab:coreset}
\end{table}

\vspace{0.05in}

\textbf{(3)} Due to the rapid increase of real-world data volume, the study on distributed computing has received a great amount of attention. Several distributed algorithms for $k$-center clustering with outliers were proposed recently~\citep{malkomes2015fast,guha2017distributed,DBLP:journals/corr/abs-1802-09205,li2018distributed}, but most of them have large approximation ratios, e.g., the algorithm of~\citet{li2018distributed} has the approximation ratio $>19$. Therefore, it is necessary to develop a communication-efficient {\em composable coreset}~\citep{DBLP:conf/pods/IndykMMM14} so that one can compute an approximate solution with higher accuracy in the central central server. Namely, the input data is partitioned to be stored in $s$ sites, and each site can compute an individual coreset and send it to the central central server; finally, the central server  computes an approximation result on the union of the collected coresets.  
Let $B$ be the information encoding a point. A straightforward implementation of our proposed coreset of Section~\ref{sec-doubling}  yields a communication cost $s\Big(2z+O\big((\frac{2}{\mu})^\rho k\big)\Big)B$, which can be too high if $z$ is large (e.g., if $z=5\%n$ and $s=10$, the cost can be larger than $nB$). In Section~\ref{sec-distriubted}, we prove that the communication cost can be reduced to be (roughly) $\Big(4z+s\cdot O\big((\frac{2}{\mu})^\rho k\big)\Big)B$ by using the partial clustering framework of \cite{guha2017distributed}; so we reduce the item ``$2sz$'' to be ``$4z$''. To the best of our knowledge, this is the first communication-efficient composable coreset for $k$-center clustering with outliers that guarantees a  $(1+ O(\mu))$-approximation error. Please see Table \ref{tab:contribution-distributed} for details.

\begin{table}[]
\resizebox{\textwidth}{!}{%
\begin{tabular}{cl|c|c|c|c}
\hline
\multicolumn{2}{l|}{} &
  \textbf{Approx.} &
  \textbf{Total Comm. (B)} &
  \textbf{Rounds} &
  \textbf{Local Time} \\ \hline
\multicolumn{2}{c|}{\cite{malkomes2015fast}} &
  $3\alpha+2$ &
  $s(k+z)$ &
  1 &
  \begin{tabular}[c]{@{}c@{}}$O\left(\left(k+z\right)n_i\right)$ \end{tabular} \\ \hline
\multicolumn{2}{c|}{\cite{guha2017distributed}} &
  $5\alpha+4$ &
  $sk+z$ &
  2 &
  \begin{tabular}[c]{@{}c@{}}$O((k+z)n_i)$\end{tabular} \\ \hline
\multicolumn{2}{c|}{\cite{li2018distributed}} &
  $((5\alpha+4)(1+\mu),1+\mu)$ &
  $O\left(\frac{sk}{\mu} \cdot \frac{\log \Delta}{\mu}\right)$ &
  2 &
  \begin{tabular}[c]{@{}c@{}}$O\left(n_i^2 \cdot \frac{\log \Delta}{\mu}\right)$\end{tabular} \\ \hline
\multicolumn{2}{c|}{\begin{tabular}[c]{@{}c@{}}\cite{DBLP:journals/corr/abs-1802-09205}\\ Deterministic\end{tabular}} &
  \multirow{2}{*}{$3+\mu$} &
  $s(k+z)(\frac{24}{\mu})^\rho $ &
  1 &
  \begin{tabular}[c]{@{}c@{}}$O((k+z)n_i(\frac{24}{\mu})^\rho)$ \end{tabular} \\ \cline{1-2} \cline{4-6} 
\multicolumn{2}{c|}{\begin{tabular}[c]{@{}c@{}}\cite{DBLP:journals/corr/abs-1802-09205}\\ Randomized\end{tabular}} &
   &
  $(sk+6z+s\log n)(\frac{24}{\mu})^\rho $ &
  1 &
  \begin{tabular}[c]{@{}c@{}}$\tilde O((k + z/s)n_i(\frac{24}{\mu})^\rho)$\end{tabular} \\ \hline
\multicolumn{2}{c|}{This paper (Theorem~\ref{the-distributed})} &
  $\alpha\times \frac{1+2\mu}{1-2\mu}=\alpha\times \big(1+O(\mu)\big)$ &
  $4z+\tilde O((sk)(\frac{2}{\mu})^\rho)$ &
  2 &
  \begin{tabular}[c]{@{}c@{}}$\tilde O\left(k\left(\frac{2}{\mu}\right)^\rho n_i \log _2 z\right)$\end{tabular} \\ \hline
\end{tabular}%
}
\captionsetup{format=hang}
\caption{Existing and our results for distributed $k$-center clustering with $z$ outliers. The ``local  time'' column illustrates the running time on each site, where $n_i$ is the data size in site $i$.  
``$\Delta$'' and ``$\rho$'' stand for the aspect ratio and the doubling dimension, respectively. ``$\alpha$'' is the approximation ratio of the algorithm run on the union of the collected coresets in the central server (e.g., if we run the $3$-approximation algorithm of~\citealt{charikar2001algorithms}, $\alpha=3$). The result of \cite{li2018distributed} is a bi-criteria approximation that discards $(1+\mu)z$ outliers.}
\label{tab:contribution-distributed}
\end{table}

\vspace{-0.1in}
\subsection{Other Related Works}

\hspace{0.24in}\textbf{Clustering with outliers.} Besides the aforementioned prior works for $k$-center clustering with outliers, a number of results for other clustering with outliers problems were also proposed in recent years.  For example, the $k$-means/median clustering with outliers algorithms with provable guarantees have been proposed by~\citet{charikar2001algorithms,chen2008constant,krishnaswamy2018constant,friggstad2018approximation}, but they are difficult to implement due to their high complexities. The heuristic but practical  algorithms without provable guarantees have also been studied, such as \cite{chawla2013k}. 
By using the local search method, \cite{gupta2017local} provided a constant factor approximation algorithm for $k$-means clustering with outliers. Furthermore, ~\cite{NEURIPS2019_73983c01} and ~\cite{DBLP:conf/uai/DeshpandeKP20} respectively showed that the quality can be improved by modifying the $k$-means++ seeding. Other recent clustering with outliers algorithms include~\citet{DBLP:conf/nips/ChenA018,DBLP:journals/corr/abs-2003-02433,DBLP:conf/esa/ChakrabartyNS22}.

\textbf{Coresets.} 
The study on coresets was initiated by \cite{DBLP:journals/jacm/AgarwalHV04}, and the technique has been extensively applied for dealing with large-scale data sets in many different areas. For example, it can be used to  reduce the computational complexities for clustering and regression problems in machine learning~\citep{cohen2021new,munteanu2018coresets}. To handle the problems with distributed data, the techniques like ``mergeable summaries
''\citep{agarwal2013mergeable} and ``composable coresets'' \citep{DBLP:conf/pods/IndykMMM14,mirrokni2015randomized} were introduced recently. \cite{DBLP:conf/cccg/AghamolaeiG18} also considered the composable coreset in doubling metrics but their method is only for the ordinary $k$-center clustering problem (without outliers).

\vspace{-0.05in}
\section{Preliminaries}
\label{sec-pre}

We consider the problem of $k$-center with outliers in arbitrary metrics and Euclidean space $\mathbb{R}^D$. Let $(X, \mathtt{d})$ be an abstract metric, where $X$ contains $n$ vertices and $\mathtt{d}(\cdot, \cdot)$ is the distance function; with a slight abuse of notation, we also use the function $\mathtt{d}$ to denote the shortest distance between two subsets  $X_1, X_2\subseteq X$, i.e., $\mathtt{d}(X_1, X_2)=\min_{p\in X_1, q\in X_2}\mathtt{d}(p, q)$. In $\mathbb{R}^D$, we use $||p-q||$ to denote the Euclidean distance between any two points $p$ and $q$. For simplicity, we assume that the distance between any pair of vertices in $X$ can be obtained in $O(1)$ time; for the problem in Euclidean space, it takes $O(D)$ time to compute the distance between any pair of points. 
Below, we introduce several important definitions that are used throughout this paper.

\begin{definition}[$k$-Center Clustering with Outliers]
\label{def-outlier}
Given a metric $(X, \mathtt{d})$ with two positive integers $k$ and $z<n$, the $k$-center clustering with outliers problem is to find a subset $X'\subseteq X$, where $|X'|\geq n-z$, and $k$ centers $\{c_1, \cdots, c_k\}\subseteq X$, such that 
\begin{eqnarray*}
\max_{p\in X'}\min_{1\leq j\leq k}\mathtt{d}(p, c_j)
\end{eqnarray*}
is minimized. If given a set $P$ of  $n$ points in $\mathbb{R}^D$, the problem is to find a subset $P'\subseteq P$, where $|P'|\geq n-z$, and $k$ centers $\{c_1, \cdots, c_k\}\subset\mathbb{R}^D$, such that $\max_{p\in P'}\min_{1\leq j\leq k}||p-c_j||$ is minimized.
\end{definition}

In this paper, we always use $X_{\mathtt{opt}}$, a subset of $X$ with size $n-z$, to denote the subset yielding the optimal solution. Also, let $\{C_1, \cdots, C_k\}$ be the $k$ clusters forming $X_{\mathtt{opt}}$, and the resulting clustering cost be $r_{\mathtt{opt}}$; that is, each $C_j$ is covered by an individual ball with radius $r_{\mathtt{opt}}$.

Usually, the optimization problems with outliers are challenging to solve. Thus we often relax our goal and allow to remove slightly more than the pre-specified number of  outliers. Actually the same relaxation idea has been adopted by a number of works on clustering with outliers problems before~\citep{charikar2003better,huang2018epsilon,li2018distributed}. So we introduce Definition~\ref{def-relax}. 
For the sake of convenience, we describe the following Definition~\ref{def-relax} and Definition~\ref{def-coreset} only for metric space. In fact, the definitions can be easily modified for the problem in Euclidean space.

\begin{definition}[$(k,z)_{\epsilon}$-Center Clustering]
\label{def-relax}
Let $(X,\mathtt{d})$ be an instance of $k$-center clustering with $z$ outliers, and $\epsilon\geq 0$. $(k,z)_{\epsilon}$-center clustering is to find a subset $X'$ of $X$, where $|X'|\geq n-(1+\epsilon)z$, such that the corresponding clustering cost of Definition~\ref{def-outlier} on $X'$ is minimized.

\textbf{(\rmnum{1})} Given a set $A$ of cluster centers ($|A|$ could be larger than $k$), we define the clustering cost 
\begin{align*}
 \phi_{\epsilon}(X, A):=\min\big\{\max_{p\in X'}\min_{c\in A}\mathtt{d}(p, c)\mid X'\subseteq X, |X'|\geq n-(1+\epsilon)z\big\}.
\end{align*}

\textbf{(\rmnum{2})} If $|A|=k$ and $\phi_{\epsilon}(X, A)\leq\alpha r_{\mathtt{opt}}$ with $\alpha>0$\footnote{Since we discard more than $z$ outliers, it is possible to have an approximation ratio $\alpha<1$, i.e., $\phi_{\epsilon}(X, A)< r_{\mathtt{opt}}$.}, 
 the set $A$ is called an $\alpha$-approximation; if $|A|=\beta k$ with $\beta> 1$, the set $A$ is called an $(\alpha, \beta)$-approximation.

\end{definition}

Obviously, the problem in Definition~\ref{def-outlier} is a special case of $(k,z)_{\epsilon}$-center clustering with $\epsilon=0$. 
Also, Definition~\ref{def-outlier} and Definition~\ref{def-relax} can be naturally extended to the \textbf{weighted case:} each vertex $p$ has a non-negative weight $w_p$ and the total weight of outliers should be equal to $z$. 
Then we have the following definition for coreset. 
\begin{definition}[Coreset]
\label{def-coreset}
Given a small parameter $\mu\in(0,1)$ and an instance $(X,\mathtt{d})$ of  $k$-center clustering with $z$ outliers, a set $S\subseteq X$ is called a $\mu$-coreset of $X$, if each vertex of $S$ is assigned a non-negative weight and $\phi_{0}(S, H)\in (1\pm\mu)\phi_{0}(X, H)$ for any set $H\subseteq X$ of $k$ vertices.
\end{definition}

Given a large-scale instance $(X, \mathtt{d})$, we can run an existing algorithm on its coreset $S$ to compute an approximate solution for $X$. If $|S|\ll n$, the  running time can be significantly reduced. Formally, we have the following claim (see the proof in Section~\ref{sec-proof-c1}).

\begin{claim}
\label{pro-core}
If the set $H$ yields an $\alpha$-approximation of the $\mu$-coreset $S$, it yields an $\alpha\times\frac{1+\mu}{1-\mu}$-approximation of $X$.
\end{claim}

As mentioned before, we also consider the case with low doubling dimension. Roughly speaking, the doubling dimension describes the expansion rate of the metric. 
For any $p\in X$ and $r\geq 0$, we use $\mathtt{Ball}(p, r)$ to denote the ball centered at $p$ with radius $r$.

\begin{definition}[Doubling Dimension]
\label{def-dd}
The doubling dimension of a metric $(X, \mathtt{d})$ is the smallest number $\rho>0$, such that for any $p\in X$ and $r\geq 0$, $X\cap \mathtt{Ball}(p, 2r)$ is always covered by the union of at most $2^\rho$ balls with radius $r$.
\end{definition}

\textbf{The rest of this paper is organized as follows.} In Section \ref{sec-center}, we present our constant factor approximations for $k$-center clustering with outliers, where the main idea is a randomized greedy approach based on the Gonzalez's algorithm. 
In Section \ref{sec-coreset}, we study the coreset for $k$-center clustering with outliers in Euclidean space and doubling metrics, respectively. In Section \ref{sec-distriubted}, we show that our proposed coreset in Section \ref{sec-coreset} can be constructed by  
a communication-efficient way for distributed setting. In Section \ref{sec-exp}, we conduct the experiments to evaluate the proposed methods.

\section{Randomized Greedy Algorithms for $(k,z)_{\epsilon}$-Center Clustering}
\label{sec-center}

For the sake of completeness, we briefly introduce the algorithm of \cite{gonzalez1985clustering} for ordinary $k$-center clustering first. Initially, it arbitrarily selects a vertex from $X$, and iteratively selects the following $k-1$ vertices, where each $j$-th step ($2\leq j\leq k$) chooses the vertex having the largest minimum distance to the already selected $j-1$ vertices; finally, each input vertex is assigned to its nearest neighbor of these selected $k$ vertices. This greedy strategy yields a $2$-approximation of $k$-center clustering; the algorithm also works for the problem in Euclidean space and yields the same approximation ratio. In this section, we show that a randomized version of the Gonzalez's algorithm can solve the  $(k,z)_{\epsilon}$-center clustering problem with quality guarantee.

\subsection{$(2, O(\frac{1}{\epsilon}))$-Approximation}
\label{sec-center-bi}

We consider the bi-criteria approximation that returns more than $k$ cluster centers. Our high-level idea is as follows.

The main challenge for implementing the Gonzalez's algorithm is that the outliers and inliers are mixed in $X$. For example, the selected vertex, which has the largest minimum distance to the already selected vertices, is very likely to be an outlier, and then  the clustering quality could be arbitrarily bad. To resolve this issue, we replace each greedy selection step by a bi-level ``greedy selection$+$random sampling'' step: select the farthest $(1+\epsilon)z$ points (rather than the farthest single point) with a small parameter $\epsilon\in (0,1)$, and then take a random sample from this selected set. Such a combined strategy can guarantee us to successfully sample a sufficient number of inliers from the $k$ optimal clusters, and meanwhile  restrict the number of sampled outliers. 
We implement our idea in Algorithm~\ref{alg-bi}. For simplicity, let $\gamma$ denote $z/n$ in the algorithm.

\begin{theorem}
\label{the-biapprox}
Let $\epsilon>0$ and $\eta\in(0,1/2)$. If we set $t=\frac{ck}{1-\eta}$  with $c=2+\frac{2}{k(1-\eta)}\ln\frac{1}{\eta}$ in  Algorithm~\ref{alg-bi}, with probability at least $1-2\eta$, $\phi_{\epsilon}(X,E)\leq 2 r_{\mathtt{opt}}$.

\end{theorem}

If $\frac{1}{\eta}$ and $\frac{1}{1-\gamma}$ are constant numbers, the size $|E|=\frac{1}{1-\gamma}\log\frac{1}{\eta}+\frac{1+\epsilon}{\epsilon}\log\frac{1}{\eta}\times t=O(\frac{k}{\epsilon})$. So Theorem~\ref{the-biapprox} implies that  $E$ is a $\big(2, O(\frac{1}{\epsilon})\big)$-approximation for  $(k,z)_{\epsilon}$-center clustering of $X$ with probability at least $1-2\eta$. 
To prove Theorem~\ref{the-biapprox}, we need Lemma~\ref{lem-select1} and Lemma~\ref{lem-select2} first. 

\begin{algorithm}[tb]
   \caption{Bi-criteria Approximation Algorithm}
   \label{alg-bi}
\begin{algorithmic}
  \STATE {\bfseries Input:} An instance $(X, \mathtt{d})$ of metric $k$-center clustering with $z$ outliers, and $|X|=n$; parameters $\epsilon>0$, $\eta\in (0,1/2)$, and $t\in\mathbb{Z}^+$.
   \STATE
\begin{enumerate}
\item Let $\gamma=z/n$ and initialize a set $E=\emptyset$. 

\item Initially, $j=1$; randomly select $\frac{1}{1-\gamma}\log\frac{1}{\eta}$ vertices from $X$ and add them to $E$.
\item Run the following steps until $j= t$:
\begin{enumerate}
\item Update $j=j+1$ and let $Q_j$ be the subset of $X$ that are the farthest $(1+\epsilon)z$ vertices to $E$ (for each vertex $p\in X$, its distance to $E$ is defined as $\min_{q\in E}\mathtt{d}(p, q)$). 
\item Randomly select $\frac{1+\epsilon}{\epsilon}\log\frac{1}{\eta}$ vertices from $Q_j$ and add them to $E$.
\end{enumerate}
\end{enumerate}
  \STATE {\bfseries Output} $E$.
\end{algorithmic}
\end{algorithm}

\begin{lemma}
\label{lem-select1}
With probability at least $1-\eta$, the set $E$ in Step 2 of Algorithm~\ref{alg-bi} contains at least one point from $X_{\mathtt{opt}}$.
\end{lemma}

Since $|X_{\mathtt{opt}}|/|X|= 1-\gamma$, Lemma~\ref{lem-select1} can be easily obtained by the following claim. 
\begin{claim}
\label{pro-sample}
Let $U$ be a set of elements and $V\subseteq U$ with $\frac{|V|}{|U|}=\tau>0$. Given $\eta\in(0,1)$, if one randomly samples $\frac{1}{\tau}\log\frac{1}{\eta}$ elements from $U$, with probability at least $1-\eta$, the sample contains at least one element from $V$.
\end{claim}
Actually Claim~\ref{pro-sample} is a folklore result that has been presented in several papers before (such as~\citealt{DBLP:conf/compgeom/DingX14}). Since each sampled element falls in $V$ with probability $\tau$, we know that the sample $S$ contains at least one element from $V$ with probability $1-(1-\tau)^{|S|}$. Therefore, if we want $1-(1-\tau)^{|S|}\geq 1-\eta$, $|S|$ should be at least $\frac{\log 1/\eta}{\log 1/(1-\tau)}\leq\frac{1}{\tau}\log\frac{1}{\eta}$.

Recall that $\{C_1, C_2, \cdots, C_k\}$ are the $k$ clusters forming $X_{\mathtt{opt}}$. Denote by $\lambda_j(E)$ the number of the clusters which have \textbf{non-empty} intersection with $E$ at the beginning of $j$-th round in Step~3 of Algorithm~\ref{alg-bi}. For example, through Lemma~\ref{lem-select1} we know that $\lambda_1(E)$ should be at least $ 1$. Obviously, if $\lambda_j(E)=k$, i.e., $C_l\cap E\neq\emptyset$ for any $1\leq l\leq k$, $E$ will yield a $2$-approximate solution by using the triangle inequality.
\begin{claim}
\label{cla-e2}
If $\lambda_j(E)=k$,  then $\phi_0(X, E)\leq 2 r_{\mathtt{opt}}$.
\end{claim}

\begin{lemma}
\label{lem-select2}
In each round of Step 3 of Algorithm~\ref{alg-bi}, either the event (1) $\mathtt{d}(Q_j, E)\leq 2 r_{\mathtt{opt}}$ happens, or with probability at least $1-\eta$, the event (2) $\lambda_j(E)\geq\lambda_{j-1}(E)+1$ happens.
\end{lemma}
\begin{proof} 
Suppose that the event (1) does not happen, i.e., $d(Q_j, E)> 2 r_{opt}$, and then we prove that the event (2) should happen with probability at least $1-\eta$. Let $\mathcal{J}$ include all the indices $l\in\{1, 2, \cdots, k\}$ with $ E\cap C_l\neq\emptyset$. We claim that $Q_j\cap C_l=\emptyset$ for each $l\in \mathcal{J}$. Otherwise, we arbitrarily select $p\in Q_j\cap C_l$ and $p'\in E\cap C_l$; by using the triangle inequality, we know that $\mathtt{d}(p,p')\leq 2 r_{\mathtt{opt}}$ which is in contradiction to the assumption $\mathtt{d}(Q_j, E)> 2 r_{\mathtt{opt}}$. Thus, $Q_j\cap X_{\mathtt{opt}}$ only contains the vertices from $C_l$ with $l\notin \mathcal{J}$. Note that the number of outliers is $z$. So we have $|Q_j\setminus X_{\mathtt{opt}}|\leq z$ and $\frac{|Q_j\cap X_{\mathtt{opt}}|}{|Q_j|}\geq \frac{\epsilon}{1+\epsilon}$. By Claim~\ref{pro-sample}, if randomly selecting $\frac{1+\epsilon}{\epsilon}\log\frac{1}{\eta}$ vertices from $Q_j$, with probability at least $1-\eta$, the sample contains at least one vertex from $Q_j\cap X_{\mathtt{opt}}$; also, the vertex must come from $\cup_{l\notin \mathcal{J}}C_l$. That is, the event (2) $\lambda_j(E)\geq\lambda_{j-1}(E)+1$ happens.
 \end{proof}

If the event (1) of Lemma~\ref{lem-select2} happens, i.e., $\mathtt{d}(Q_j, E)\leq 2 r_{\mathtt{opt}}$, then it implies that 
\begin{align*}
\max_{p\in X\setminus Q_j}\mathtt{d}(p, E)\leq 2 r_{\mathtt{opt}}; 
\end{align*}
moreover, since $|Q_j|=(1+\epsilon)z$, we have $\phi_{\epsilon}(X,E)\leq 2 r_{\mathtt{opt}}$. 
Next, we assume that the event (1) in Lemma~\ref{lem-select2} never happens, and prove that $\lambda_j(E)=k$ with constant probability when $j=\Theta(k)$. The following idea  is inspired from  \citet{aggarwal2009adaptive} which  achieves a bi-criteria approximation for $k$-means clustering. We define a random variable $x_j$: $x_j=1$ if $\lambda_j(E)=\lambda_{j-1}(E)$, or $x_j=0$ if $\lambda_j(E)\geq\lambda_{j-1}(E)+1$, for $j=1, 2, \cdots$. So $\mathbb{E}[x_j]\leq\eta$ by Lemma~\ref{lem-select2} and
\begin{align}
\sum_{1\leq s\leq j}(1-x_s)\leq\lambda_j(E). \label{for-azuma2}
\end{align}
Also, let $J_j=\sum_{1\leq s\leq j}(x_s-\eta)$ and $J_0=0$. Then, $\{J_0, J_1, J_2, \cdots\}$ is a super-martingale with $J_{j+1}-J_j< 1$. Through the {\em Azuma-Hoeffding inequality}~\citep{alon2004probabilistic}, we have 
$\mathtt{Prob}(J_t\geq J_0+h)\leq e^{-\frac{h^2}{2t}}$ 
for any $t\in\mathbb{Z}^+$ and $h>0$. Let $t=\frac{ck}{1-\eta}$  with $c=2+\frac{2}{k(1-\eta)}\ln\frac{1}{\eta}$ and $h=(c-1)k$, the inequality implies 
\begin{eqnarray}
&&\mathtt{Prob}(\sum_{1\leq s\leq t}(1-x_s)\geq t(1-\eta)-h)\geq 1-e^{-\frac{h^2}{2t}}\nonumber\\
 \Longrightarrow && \mathtt{Prob}(\sum_{1\leq s\leq t}(1-x_s)\geq k)\geq 1-e^{-\frac{k(c-1)^2(1-\eta)}{2c}}\geq 1- e^{-(c/2-1)k(1-\eta)}\nonumber\\
\Longrightarrow &&\mathtt{Prob}(\sum_{1\leq s\leq t}(1-x_s)\geq k)\geq 1-\eta. \label{for-azuma}
\end{eqnarray}
Combining (\ref{for-azuma2}) and (\ref{for-azuma}), we know that $\lambda_t(E)\geq k$ with probability at least $1-\eta$. Moreover, when $\lambda_t(E)=k$, it is easy to  know that $E$ is a $2$-approximate solution by Claim~\ref{cla-e2}. Together with Lemma~\ref{lem-select1}, we immediately have Theorem~\ref{the-biapprox} where the overall success probability is at least $(1-\eta)^2>1-2\eta$.

\textbf{Time complexity.}  In each round of Step~3, there are $O(\frac{1}{\epsilon})$ new vertices added to $E$, thus it takes $O(\frac{1}{\epsilon}n)$ time to update the distances from the vertices of $X$ to $E$; to select the set $Q_j$, we can apply the linear time selection algorithm of \citet{blum1973time}. Overall, the running time of Algorithm~\ref{alg-bi} is $O(\frac{k}{\epsilon}n)$. If the given instance is in $\mathbb{R}^D$, the running time will be $O(\frac{k}{\epsilon}n D)$.

\subsection{$2$-Approximation for Constant $k$}
\label{sec-center-single}
If $k$ is a constant number, we show that a single-criterion $2$-approximation can be achieved. Actually, we use the same strategy as Section~\ref{sec-center-bi}, but only run $k$ rounds with each round sampling only one vertex. See Algorithm~\ref{alg-single} for the details.

\begin{algorithm}[tb]
   \caption{$2$-Approximation Algorithm}
   \label{alg-single}
\begin{algorithmic}
  \STATE {\bfseries Input:} An instance $(X,\mathtt{d})$ of metric $k$-center clustering with $z$ outliers, and $|X|=n$; a parameter $\epsilon>0$.
   \STATE
\begin{enumerate}
\item Initialize a set $E=\emptyset$.

\item Let $j=1$; randomly select one vertex from $X$ and add it to $E$.
\item Run the following steps until $j= k$:
\begin{enumerate}
\item Update $j=j+1$ and let $Q_j$ be the subset of $X$ that are the farthest $(1+\epsilon)z$ vertices to $E$.  
\item Randomly select one vertex from $Q_j$ and add it to $E$.
\end{enumerate}

\end{enumerate}
  \STATE {\bfseries Output} $E$.
\end{algorithmic}
\end{algorithm}

Denote by $\{v_1, \cdots, v_k\}$ the $k$ sampled vertices of $E$. Actually, the proof of Theorem~\ref{the-kcenter} is similar to the analysis in Section~\ref{sec-center-bi}. The only difference is that the probability that the event (2) $\lambda_j(E)\geq\lambda_{j-1}(E)+1$ in Lemma~\ref{lem-select2} happens is changed to be at least $\frac{\epsilon}{1+\epsilon}$. Also note that $v_1\in X_{\mathtt{opt}}$ with probability $1-\gamma$ (because $\gamma=z/n$). If all of these events happen, either we obtain a $2$-approximation before $k$ steps (i.e., $\mathtt{d}(E, X\setminus Q_j)\leq 2 r_{\mathtt{opt}}$ for some $j<k$), or $\{v_1, \cdots, v_k\}$ fall into the $k$ optimal clusters $C_1, C_2, \cdots, C_k$ separately (i.e., $\lambda_k(E)=k$). No matter which case happens, we always obtain a $2$-approximation with respect to the $(k,z)_{\epsilon}$-center clustering problem. So we have the following  Theorem~\ref{the-kcenter}.

\begin{theorem}
\label{the-kcenter}
Algorithm~\ref{alg-single} returns a $2$-approximation for the problem of $(k,z)_{\epsilon}$-center clustering on $X$, with probability at least $(1-\gamma)(\frac{\epsilon}{1+\epsilon})^{k-1}$. The time complexity is $O(kn)$. If the given instance is in $\mathbb{R}^D$, the time complexity will be $O(kn D)$.
\end{theorem}

To boost the probability of Theorem~\ref{the-kcenter}, we just need to repeatedly run the algorithm. The success probability is easy to calculate by taking the union bound.

\begin{corollary}
\label{the-kcenter2}
If we run Algorithm~\ref{alg-single} $ O\big(\frac{1}{1-\gamma}(\frac{1+\epsilon}{\epsilon})^{k-1}\big)$ times, with constant probability, at least one time the algorithm  returns a $2$-approximation for the problem of $(k,z)_\epsilon$-center clustering.
\end{corollary}

 \subsection{Sublinear Time Implementation of Algorithm~\ref{alg-bi}}
 \label{sec-sublinear}
 The input data size $n$ can be quite large in practice, so in this section we consider to implement Algorithm~\ref{alg-bi} with a lower time complexity. We present a modified version of Algorithm~\ref{alg-bi} that only needs an $O(\frac{k^2}{\gamma \epsilon^2})$ time complexity, where $\gamma=\frac{z}{n}$. When the data contains heavy noise and the outliers takes a constant factor of $n$ (e.g., $z=5\%n$ and $\frac{1}{\gamma}=20$), the algorithm has a sublinear time complexity that is independent of $n$.  Moreover, the quality of Algorithm~\ref{alg-bi} presented in Theorem~\ref{the-biapprox} can be guaranteed exactly.

 \begin{algorithm}[tbp]
   \caption{Sublinear Time Implementation of Algorithm~\ref{alg-bi}}
   \label{alg-sub-bi}
\begin{algorithmic}
  \STATE {\bfseries Input:} An instance $(X, \mathtt{d})$ of metric $k$-center clustering with $z$ outliers, and $|X|=n$; parameters $\epsilon>0$, $\eta\in (0,1/2)$, and $t\in\mathbb{Z}^+$.
   \STATE
\begin{enumerate}
\item Let $\gamma=z/n, \sigma=\frac{2}{1+\sqrt{1+\frac{4(1+\epsilon)}{3\epsilon}}},n'=\frac{3}{\sigma^2(1+\epsilon)\gamma}\log\frac{4}{\eta}$ and initialize a set $E=\emptyset$. 

\item Initially, $j=1$; randomly select $\frac{1}{1-\gamma}\log\frac{1}{\eta}$ vertices from $X$ and add them to $E$.
\item Run the following steps until $j= t$:
\begin{enumerate}

\item Update $j=j+1$; uniformly sample $n'$ vertices from $X$ and denote the sampled set as $A_j$.

\item Let $\hat r_j$ be the $(1+\sigma)(1+\epsilon)\gamma n'$-th farthest distance from $A_j$ to $E$. Let $\hat{A}_j= \{p\in A_j\mid \mathtt{d}(p,E)\geq \hat r_j\}$.
\item Add $\hat{A}_j$ to $E$.
\end{enumerate}
\end{enumerate}
  \STATE {\bfseries Output} $E$.
\end{algorithmic}
\end{algorithm}
 
 \vspace{0.05in}
 \textbf{The key observation and analysis.} Recall that in each round of Algorithm~\ref{alg-bi}, we need to scan the whole data set and select the  farthest $(1+\epsilon)z$ vertices to $E$. Thus it takes linear time in each round. Our key observation is that we can actually avoid this step by   simple random sampling. The new idea is shown in Algorithm~\ref{alg-sub-bi} (Step 3). 
 We still let $Q_j$ be the set of $(1+\epsilon)z$ farthest vertices to $E$ from $X$ in the $j$-th round (as Step 3(a) in Algorithm~\ref{alg-bi}). We randomly sample $n'$ vertices from $X$ and use $A_j$ to denote this sampled set. We can view each sampled vertex of $A_j$ as an independent random variable $\in \{0,1\}$: each sampled vertex is labeled by ``$1$'' if it belongs to $Q_j$; otherwise, it is labeled by ``$0$''. Let $\sigma\in (0,1)$. Through the {\em Chernoff bound},  we have 
\begin{eqnarray}
\mathtt{Prob}\left [|A_j\cap Q_j|\in(1\pm\sigma)(1+\epsilon)\gamma n' \right ]\geq 1-2e^{-\frac{1}{3}(\sigma^2(1+\epsilon)\gamma n')}. \label{for-sublinear-1}
\end{eqnarray}
We select the farthest $(1+\sigma)(1+\epsilon)\gamma n'$ vertices to $E$ from $A_j$, where the selected set is denoted by $\hat{A}_j$. 
 We set the parameter $\sigma=\frac{2}{1+\sqrt{1+\frac{4(1+\epsilon)}{3\epsilon}}}$, and then the right hand-side of (\ref{for-sublinear-1}) becomes $1-\frac{\eta}{2}$. So with probability at least $1-\frac{\eta}{2}$, $|A_j\cap Q_j|\leq(1+\sigma)(1+\epsilon)\gamma n'$. Because $|\hat{A}_j|=(1+\sigma)(1+\epsilon)\gamma n'$, we have 
 \begin{eqnarray}
 |\hat{A}_j|\geq|A_j\cap Q_j|.\label{for-sublinear-jan-1} 
 \end{eqnarray}
 Denote by $r_j$ the distance $\mathtt{d}(A_j\cap Q_j,E)$. Since $Q_j$ is the set of $(1+\epsilon)z$ farthest vertices to $E$ from $X$, we have 
 \begin{eqnarray}
\{p\in A_j\mid \mathtt{d}(p,E)\geq r_j\}&=&A_j\cap Q_j ;\label{for-sublinear-jan-2}\\
\{p\in A_j \mid \mathtt{d}(p,E)> r_j\}&\subsetneqq& A_j\cap Q_j. \label{for-sublinear-jan-3} \end{eqnarray}
 Now we claim that $\hat r_j\leq r_j$ where $\hat r_j$ is the $(1+\sigma)(1+\epsilon)\gamma n'$-th farthest distance from $A_j$ to $E$ (as defined in Step 3(b) of Algorithm~\ref{alg-sub-bi}).  Otherwise, if $\hat{r}_j>r_j$, we can deduce that $\hat{A}_j\subseteq \{p\in A_j \mid d(p,E)> r_j\}\subsetneqq A_j\cap Q_j$ from (\ref{for-sublinear-jan-3}) the fact $\hat{A}_j=\{p\in A_j \mid d(p,E)\geq \hat r_j\}$; so we have $|\hat{A}_j|<|A_j\cap Q_j|$ which is contradictory to (\ref{for-sublinear-jan-1}).  Therefore we have $\hat r_j\leq r_j$; together with (\ref{for-sublinear-jan-2}), it implies 
 \begin{eqnarray*}
 A_j\cap Q_j \subseteq \{p\in A_j \mid \mathtt{d}(p,E)\geq \hat r_j\}=\hat{A}_j. 
 \end{eqnarray*}
Hence $A_j \cap Q_j \subseteq \hat{A}_j \cap Q_j$. On the other hand, since $\hat{A}_j\subseteq A_j$, it is easy to know $\hat{A}_j\cap Q_j\subseteq A_j\cap Q_j$. Therefore, 
 \begin{eqnarray}
 A_j\cap Q_j= \hat{A}_j\cap Q_j.\label{for-sublinear-2}
 \end{eqnarray}
 Moreover, from (\ref{for-sublinear-1}) again we know that 
 \begin{eqnarray}
 |A_j\cap Q_j|\geq(1-\sigma)(1+\epsilon)\gamma n'=\frac{1+\epsilon}{\epsilon}\log \frac{2}{\eta} \label{for-sublinear-3}
 \end{eqnarray}
with probability at least $1-\frac{\eta}{2}$, where we set $n'=\frac{3}{\sigma^2(1+\epsilon)\gamma}\log\frac{4}{\eta}$. From (\ref{for-sublinear-2}) and (\ref{for-sublinear-3}), we know that 
\begin{eqnarray*}
|\hat{A}_j\cap Q_j|=|A_j\cap Q_j|\geq\frac{1+\epsilon}{\epsilon}\log \frac{2}{\eta}.
\end{eqnarray*} 
Therefore, $\hat{A}_j$ contains at least $\frac{1+\epsilon}{\epsilon}\log \frac{2}{\eta}$ vertices from $Q_j$. 
Then we can obtain the similar result as Lemma~\ref{lem-select2}:  in each round, the event `` either (1)  $\mathtt{d}(Q_j, E)\leq 2 r_{\mathtt{opt}}$ or (2) $\lambda_j(E)\geq\lambda_{j-1}(E)+1$'' happens with probability at least $1-\frac{\eta}{2}-\frac{\eta}{2}= 1-\eta$ (recall the probability in (\ref{for-sublinear-1}) is $1-\frac{\eta}{2}$, so the overall probability is at least $1-\frac{\eta}{2}-\frac{\eta}{2}$). 

Together with the same super-martingale argument of Theorem \ref{the-biapprox}, we have the following result.

\begin{theorem}
\label{the-sub_biapprox}
Let $\epsilon>0$. If we set $t=\frac{ck}{1-\eta}$ with $c=2+\frac{2}{k(1-\eta)}\ln\frac{1}{\eta}$ for  Algorithm~\ref{alg-sub-bi}, with probability at least $1-2\eta$, $\phi_{\epsilon}(X,E)\leq 2 r_{\mathtt{opt}}$. 
\end{theorem}
\textbf{Quality and time complexity.} 
We assume that $\gamma$ and $1/\eta$ are constants.  Algorithm~\ref{alg-sub-bi} adds $O(\frac{1}{\sigma^2})$ vertices to $E$ at each iteration. Note that $\sigma=\Theta(\sqrt{\epsilon})$, which implies the number of vertices added to $E$ at each iteration is $O(\frac{1}{\epsilon})$.  So $|E|=O(\frac{k}{\epsilon})$ at the end of the algorithm. Then Theorem~\ref{the-sub_biapprox} implies that $E$ is a $\left(2,O(\frac{1}{\epsilon})\right)$-approximation for $(k,z)_{\epsilon}$-center clustering of $X$ with constant probability. Each round we compute the distances from the vertices of $A_j$ to $E$ and select $\hat A_j$. Since $|A_j|=O(\frac{1}{\gamma\epsilon})$ and $|E|= O(\frac{k}{\epsilon})$, we have the time complexity $O(\frac{k}{\gamma\epsilon^2})$ for computing the distances in each round of Algorithm~\ref{alg-sub-bi}. The selection of $\hat A_j$ takes $O(|A_j|)$ time. Overall, the time complexity of Algorithm~\ref{alg-sub-bi} is $O(\frac{k^2}{\gamma\epsilon^2})$, which is independent of $n$. If the given distance is in $\mathbb{R}^D$, the time complexity will be $O(\frac{k^2}{\gamma\epsilon^2}D)$. If the input  data size $n$ is large, Algorithm~\ref{alg-sub-bi} can significantly reduce the time complexity and meanwhile preserve the same clustering quality of Algorithm~\ref{alg-bi}.

\section{Coresets for $k$-Center Clustering with Outliers}
\label{sec-coreset}
In this section, we consider the coreset construction problem for $k$-center clustering with outliers. First, we show that the simple uniform sampling approach can yield a slightly weaker coreset for $(k, z)_{\epsilon}$-center clustering in Euclidean space, where the number of discarded outliers is amplified from $(1+\epsilon)z$ to be $(1+O\big(\epsilon)\big)z$. Then we consider the coreset construction in doubling metrics. We show that the idea of Algorithm~\ref{alg-bi} can be extended for building the coreset efficiently, even if the doubling dimension is not given.

\subsection{Uniform Sampling in Euclidean Space}
\label{sec-core1}

Given a metric $(X,\mathtt{d})$, \cite{charikar2003better} showed that we can use a random sample $S$ to replace $X$. Recall $\gamma=z/n$. Let $|S|=O(\frac{k}{\epsilon^2 \gamma}\ln n)$ and $E$ be an $\alpha$-approximate solution of $(k, z)_\epsilon$-center clustering on $(S,\mathtt{d})$, then $E$ is an $\alpha$-approximate solution of $(k, z)_{O(\epsilon)}$-center clustering on $(X,d)$ with constant probability. In a $D$-dimensional Euclidean space, \cite{huang2018epsilon} showed a similar result, where the sample size $|S|=\tilde{O}(\frac{1}{\epsilon^2\gamma^2}kD)$\footnote{The asymptotic notation $\tilde{O}(f)=O\big(f\cdot \mathtt{polylog}(\frac{kD}{\epsilon\gamma})\big)$.}. In this section, we show that the sample size of \citet{huang2018epsilon} can be further improved by a factor $\frac{1}{\gamma}$ and the new sample size is $\tilde{O}(\frac{1}{\epsilon^2\gamma}kD)$. This improvement could be important for the case $z\ll n$, e.g., $z=\sqrt{n}$. Below We revisit their idea first, and then provide a more careful analysis to achieve the improvement.

Let $P$ be a set of $n$ points in $\mathbb{R}^D$. Consider the range space $\Sigma=(P, \Pi)$ where each range $\pi\in \Pi$ is the complement of union of $k$ balls in $\mathbb{R}^D$. We know that the VC dimension of balls is $O(D)$~\citep{alon2004probabilistic}, and therefore the VC dimension of union of $k$ balls is $O(kD \log k)$~\citep{blumer1989learnability}. That is, the VC dimension of the range space $\Sigma$ is $O(kD \log k)$.
Let $\epsilon\in(0,1)$, and an ``$\epsilon$-sample'' $S$ of $P$ is defined as follows: 
\begin{eqnarray*}
\forall \pi\in\Pi,\hspace{0.2in} \big|\frac{|\pi\cap P|}{|P|}-\frac{|\pi\cap S|}{|S|}\big|\leq \epsilon.
\end{eqnarray*}
Roughly speaking, $S$ is an approximation of $P$ with an additive error within each range $\pi$. 
Given a range space with the VC dimension $d_{\mathtt{vc}}$, an $\epsilon$-sample can be easily obtained via uniform sampling~\citep{alon2004probabilistic}, where the success probability is $1-\lambda$ and the sample size is $O\big(\frac{1}{\epsilon^2}(d_{\mathtt{vc}}\log\frac{d_{\mathtt{vc}}}{\epsilon}+\log\frac{1}{\lambda})\big)$ for any $0<\lambda<1$.
For our problem, we need to replace the ``$\epsilon$'' of the ``$\epsilon$-sample'' by $\epsilon\gamma$ to guarantee that the number of uncovered points is bounded by $\big(1+O(\epsilon)\big)\gamma n$ (we show the details below). Since $d_{\mathtt{vc}}=O(kD \log k)$, the sample size is $\tilde{O}(\frac{1}{\epsilon^2\gamma^2}kD)$~\citep{huang2018epsilon}.

Actually, the front factor $\frac{1}{\epsilon^2\gamma^2}$ of the sample size can be further reduced to be $\frac{1}{\epsilon^2\gamma}$ by a more careful analysis. We observe that there is no need to guarantee the additive error for each range $\pi$ (as the definition of $\epsilon$-sample). Instead, only a multiplicative error for the ranges covering at least $\gamma n$ points should be sufficient. Note that when a range covers more points, the multiplicative error is weaker than the additive error and thus the sample size is reduced. For this purpose, we use the {\em relative approximation}~\citep{har2011relative,li2001improved}: let $S\subseteq P$ be a subset of size $\tilde{O}(\frac{1}{\epsilon^2\gamma}kD)$ chosen uniformly at random, then with constant probability,
\begin{align}
\forall \pi\in\Pi,\ \Big|\frac{|\pi\cap P|}{|P|}-\frac{|\pi\cap S|}{|S|}\Big|\leq \epsilon\times\max\Big\{\frac{|\pi\cap P|}{|P|}, \gamma\Big\}. \label{for-relativesample}
\end{align}
We formally state our result below. Theorem~\ref{the-samplereduce} shows that if we have an $\alpha$-approximation algorithm, we can run it on the sample $S$ to obtain a solution $E$, which is also an $\alpha$-approximate solution for $(k, z)_{O(\epsilon)}$-center clustering on $P$. Because $|S|\ll |P|$, we can reduce a great amount of runtime.

\begin{theorem}
\label{the-samplereduce}
Let $P$ be an instance for the problem of $k$-center clustering with outliers in $\mathbb{R}^{D}$ as described in Definition~\ref{def-outlier}, and $S\subseteq P$ be a subset of size $\tilde{O}(\frac{1}{\epsilon^2\gamma}kD)$ chosen uniformly at random. Suppose $\epsilon\leq 0.5$. Let $S$ be a new instance for the problem of $k$-center clustering with outliers where the number of outliers is set to be $z'=(1+\epsilon)\gamma |S|$. If $E$ is an $\alpha$-approximate solution of $(k, z')_{\epsilon}$-center clustering on $S$,  then $E$ is an $\alpha$-approximate solution of $(k, z)_{O(\epsilon)}$-center clustering on $P$, with constant probability.
\end{theorem}
\begin{proof}
We assume that $S$ is a relative approximation of $P$ and (\ref{for-relativesample}) is true (this happens with constant probability). Let $\mathbb{B}_{\mathtt{opt}}$ be the set of $k$ balls covering $(1-\gamma)n$ points induced by the optimal solution for $P$, and $\mathbb{B}_{S}$ be the set of $k$ balls induced by an $\alpha$-approximate solution of $(k, z')_{\epsilon}$-center clustering on $S$. Suppose the radius of each ball in $\mathbb{B}_{\mathtt{opt}}$ (resp., $\mathbb{B}_{S}$) is $r_{\mathtt{opt}}$ (resp., $r_S$). We denote the complements of $\mathbb{B}_{\mathtt{opt}}$ and $\mathbb{B}_{S}$ as $\pi_{\mathtt{opt}}$ and $\pi_{S}$, respectively.

First, since $\mathbb{B}_{\mathtt{opt}}$ covers $(1-\gamma)n$ points of $P$ and $S$ is a relative approximation of $P$, we have
\begin{align*}
\frac{\big|\pi_{\mathtt{opt}}\cap S\big| }{|S|}\leq  \frac{\big|\pi_{\mathtt{opt}}\cap P\big| }{|P|}+ \epsilon\times\max\Big\{\frac{|\pi_{\mathtt{opt}}\cap P|}{|P|}, \gamma\Big\}= (1+\epsilon)\gamma 
\end{align*}
by (\ref{for-relativesample}). That is, the set balls $\mathbb{B}_{\mathtt{opt}}$ cover at least $\big(1-(1+\epsilon)\gamma\big) |S|$ points of $S$, and therefore it is a feasible solution for the instance $S$ with respect to the problem of $k$-center clustering with $z'$ outliers. 
Since $\mathbb{B}_{S}$ is  an $\alpha$-approximate solution of $(k, z')_{\epsilon}$-center clustering on $S$, we have
\begin{align}
r_S \leq  \alpha r_{\mathtt{opt}};  \hspace{0.2in}
|\pi_{S}\cap S| \leq (1+\epsilon)z'=(1+\epsilon)^2\gamma|S|. \label{for-samplereduce2}
\end{align}

Now, we claim that
\begin{align}
\big|\pi_{S}\cap P\big|\leq \frac{(1+\epsilon)^2}{1-\epsilon}\gamma |P|. \label{for-samplereduce3}
\end{align}
Assume that (\ref{for-samplereduce3}) is not true, then (\ref{for-relativesample}) implies
\begin{eqnarray*}
\Big|\frac{|\pi_{S}\cap P|}{|P|}-\frac{|\pi_{S}\cap S|}{|S|}\Big|\leq \epsilon\times \max\Big\{\frac{|\pi_{S}\cap P|}{|P|},\gamma\Big\}=\epsilon \frac{|\pi_{S}\cap P|}{|P|}.
\end{eqnarray*}
So $\frac{|\pi_{S}\cap S|}{|S|}\geq (1-\epsilon)\frac{|\pi_{S}\cap P|}{|P|}>(1+\epsilon)^2\gamma$, which is in contradiction with the second inequality of (\ref{for-samplereduce2}), and thus (\ref{for-samplereduce3}) is true. We assume $\epsilon\leq 0.5$, so $\frac{1}{1-\epsilon}\leq 1+2\epsilon$ and $\frac{(1+\epsilon)^2}{1-\epsilon}=1+O(\epsilon)$. Consequently (\ref{for-samplereduce3}) and the first inequality of (\ref{for-samplereduce2}) together imply that $\mathbb{B}_{S}$ is an $\alpha$-approximate solution of $(k, z)_{O(\epsilon)}$-center clustering on $P$.
 \end{proof}

\subsection{Coreset Construction in Doubling Metrics}
\label{sec-doubling}
Actually the sample obtained in Theorem~\ref{the-samplereduce} is not a standard coreset as  Definition~\ref{def-coreset}, since it always incurs an error on the number of discarded outliers. In this section, we consider constructing the coreset that strictly satisfies Definition~\ref{def-coreset}. We introduce the following assumption first.

\begin{assumption}
\label{ass-1}
Given an instance $(X,\mathtt{d})$ of $k$-center clustering with outliers, the metric $(X_{\mathtt{opt}},\mathtt{d})$, i.e., the metric formed by the set of inliers, has a constant doubling dimension $\rho>0$.
\end{assumption}

 We do not have any restriction on the outliers $X\setminus X_{\mathtt{opt}}$. Thus the above assumption is more relaxed and practical than assuming the whole $(X,\mathtt{d})$ has a constant doubling dimension (e.g., the previous coreset construction algorithm of \cite{DBLP:journals/corr/abs-1802-09205} assumed that the whole $(X,\mathtt{d})$ has a constant doubling dimension $\rho$). 
From Definition~\ref{def-dd}, we directly know that each optimal cluster $C_j$ of $X_{\mathtt{opt}}$ can be covered by $2^\rho$ balls with radius $r_{\mathtt{opt}}/2$ (see the left figure in Figure~\ref{fig-dd}). So we can imagine that the instance $(X, \mathtt{d})$ has $2^\rho k$ clusters, where the optimal radius is at most $r_{\mathtt{opt}}/2$. Therefore, we can just replace $k$ by $2^\rho k$ in Algorithm~\ref{alg-bi}, so as to reduce the approximation ratio (i.e., the ratio of the obtained radius to $r_{\mathtt{opt}}$) from $2$ to $1$.

\begin{figure}[ht]
\begin{center}
    \includegraphics[height=1.5in]{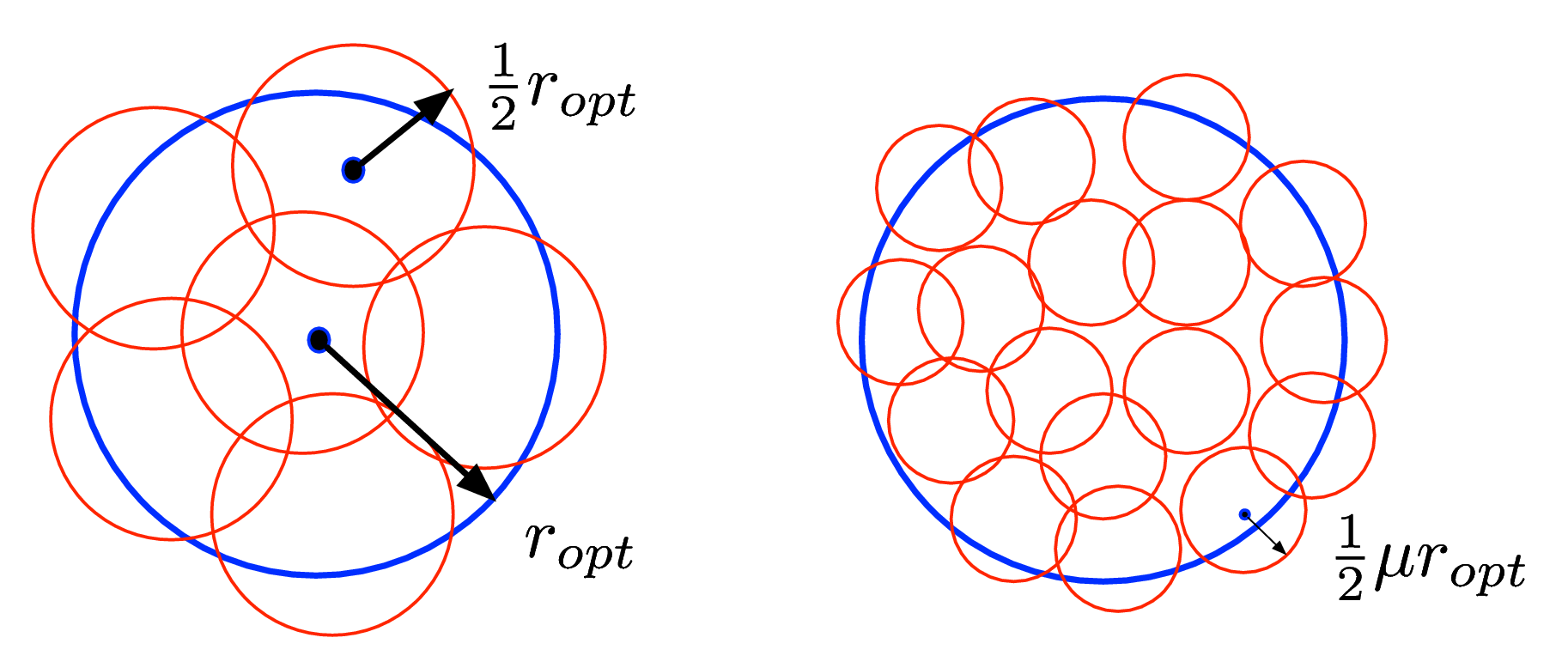}  
    \end{center}
  \caption{Illustrations for Theorem~\ref{the-double-biapprox} and Theorem~\ref{the-coreset}.}   
   \label{fig-dd}
\end{figure}

\begin{theorem}
\label{the-double-biapprox}
If we set $t=\frac{2^\rho ck}{1-\eta}$ with $c=2+\frac{2}{k(1-\eta)}\ln\frac{1}{\eta}$  for Algorithm~\ref{alg-bi}, with probability at least $1-2\eta$, $\phi_{\epsilon}(X,E)\leq  r_{\mathtt{opt}}$. So the set $E$ is a $\big(1, O(\frac{2^\rho}{\epsilon})\big)$-approximation for the problem of $(k,z)_{\epsilon}$-center clustering, and the time complexity is $O( (k+\ln\frac{1}{2\eta})\frac{2^\rho}{\epsilon}n\ln \frac{1}{2\eta})$.
\end{theorem}

Theorem~\ref{the-double-biapprox} is a warm-up, and we can further construct the coreset for $k$-center clustering with outliers. Let $\mu\in (0,1)$, and for simplicity we assume that $\log \frac{2}{\mu}$ is an integer. If applying Definition~\ref{def-dd} recursively, we know that each $C_j$ is covered by $2^{\rho\log 2/\mu}=(\frac{2}{\mu})^\rho$ balls with radius $\frac{\mu}{2} r_{\mathtt{opt}}$, and $X_{\mathtt{opt}}$ is covered by $(\frac{2}{\mu})^\rho k$ such balls in total. See the right figure in Figure~\ref{fig-dd}. Then we have Algorithm~\ref{alg-coreset} based on this observation.

\begin{algorithm}[tb]
   \caption{Coreset Construction in Doubling Metrics}
   \label{alg-coreset}
\begin{algorithmic}
  \STATE {\bfseries Input:} An instance $(X,d)$ of metric $k$-center clustering with $z$ outliers, and $|X|=n$; parameters $\eta\in(0,1/2)$ and $\mu\in (0,1)$.
   \STATE
\begin{enumerate}
\item Let $l=(\frac{2}{\mu})^\rho k$, $c=2+\frac{2}{k(1-\eta)}\ln\frac{1}{\eta}$.

\item Set $\epsilon=1$ and run Algorithm~\ref{alg-bi} with $t=\frac{cl}{1-\eta}$ rounds.
Denote by $\tilde{r}=\phi_1(X,E)$ the maximum distance between $E$ and $X$ by excluding the farthest $2z$ vertices, after the final round of Algorithm~\ref{alg-bi}.

\item Let 
$X_{\tilde{r}}=\{p\mid p\in X \text{ and } \mathtt{d}(x, E)\leq \tilde{r}\}$.

\item For each vertex $p\in X_{\tilde{r}}$, assign it to its nearest neighbor in $E$; for each vertex $q\in E$, let its weight be the number of vertices assigning to it.

\item Add $X\setminus X_{\tilde{r}}$ to $E$;  each vertex of $X\setminus X_{\tilde{r}}$ has weight $1$.

\end{enumerate}
  \STATE {\bfseries Output} $E$ as the coreset.
\end{algorithmic}
\end{algorithm}

\begin{theorem}
\label{the-coreset}
	Let $\eta\in (0,1/2)$. With probability at least $1-2\eta$, Algorithm~\ref{alg-coreset} returns a $\mu$-coreset $E$ of $k$-center clustering with $z$ outliers. The size of $E$ is at most $2z+ O\big((\frac{2}{\mu})^\rho (k+\ln\frac{1}{2\eta})\ln\frac{1}{2\eta}\big)$, and the construction time is $O(n(\frac{2}{\mu})^\rho (k+\ln\frac{1}{2\eta})\ln\frac{1}{2\eta})$.
\end{theorem}
\begin{proof}
Similar to Theorem~\ref{the-double-biapprox}, we know that  $|X_{\tilde{r}}|= n-2z$ and $\tilde{r}\leq 2\times \frac{\mu}{2} r_{\mathtt{opt}}=\mu r_{\mathtt{opt}}$ with probability at least $1-2\eta$. The size of $E$ is 
$$|X\setminus X_{\tilde{r}}|+O\big((\frac{2}{\mu})^\rho (k+\ln{\frac{1}{2\eta}})\ln \frac{1}{2\eta}\big)= 2z+O\big((\frac{2}{\mu})^\rho (k+\ln{\frac{1}{2\eta}})\ln{\frac{1}{2\eta}}\big).$$
Moreover, it is easy to see that the running time of Algorithm~\ref{alg-coreset} is $O\big((\frac{2}{\mu})^\rho (k+\ln\frac{1}{2\eta})n\ln{\frac{1}{2\eta}}\big)$.
Next, we show that $E$ is a qualified $\mu$-coreset of $X$. 

For each vertex $q\in E$, denote by $w(q)$ the weight of $q$; for the sake of convenience in our proof, we view each $q$ as a set of $w(q)$ overlapping unit weight vertices. Thus, from the construction of $E$, we can see that there is a bijective mapping $f$ between $X$ and $E$, where 
\begin{align}
\mathtt{d}\left(p,f(p)\right)\leq\tilde{r}\leq\mu r_{\mathtt{opt}}, \hspace{0.2in} \forall p\in X. \label{for-map}
\end{align}

Let $H=\{c_1, c_2, \cdots, c_k\}$ be any $k$ vertices of $X$. Suppose that $H$ induces $k$ clusters $\{A_1, A_2, \cdots, A_k\}$ (resp., $\{B_1, B_2, \cdots, B_k\}$) with respect to the problem of $k$-center clustering with $z$ outliers on $E$ (resp., $X$), where each $A_j$ (resp., $B_j$) has the cluster center $c_j$ for $1\leq j\leq k$. Let $r_E=\phi_0(E, H)$ and $r_X=\phi_0(X, H)$, respectively. Also, let $r'_E$ (resp., $r'_X$) be the smallest value $r$,  such that for any $1\leq j\leq k$, $f(B_j)\subseteq \mathtt{Ball}(c_j, r)$ (resp., $f^{-1}(A_j)\subseteq \mathtt{Ball}(c_j, r)$). We need the following claim (see the proof in Section~\ref{sec-proof-cla-core}).

\begin{claim}
\label{cla-core}
$|r'_E-r_X|\leq \mu r_{\mathtt{opt}}$ and $|r'_X-r_E|\leq \mu r_{\mathtt{opt}}$.
\end{claim}
In addition, since $\{f(B_1), \cdots, f(B_k)\}$ also form $k$ clusters for the instance $E$ with the fixed $k$ cluster centers of $H$, we know that
$r'_E\geq \phi_0(E, H)=r_E$. 
Similarly, we have 
$r'_X\geq r_X$.  
Combining Claim~\ref{cla-core}, we have 
\begin{align*}
r_X-\mu r_{\mathtt{opt}}\leq \underbrace{r'_X-\mu r_{\mathtt{opt}}\leq r_E}_{\text{by Claim~\ref{cla-core}}}\leq \underbrace{r'_E\leq r_X+\mu r_{\mathtt{opt}}}_{\text{by Claim~\ref{cla-core}}}.
\end{align*}
So $|r_X-r_E|\leq \mu r_{\mathtt{opt}}$, i.e., $\phi_0(E, H)\in \phi_0(X, H)\pm \mu r_{\mathtt{opt}}\subseteq(1\pm \mu) \phi_0(X, H)$.  Therefore $E$ is a $\mu$-coreset of $(X,\mathtt{d})$.
 \end{proof}
 \begin{remark}
\label{rem-coreset}
	\textbf{(1)} It is worth emphasizing that the uniform sampling idea in Section~\ref{sec-core1} cannot avoid the error on the number of excluded outliers; the sample size will become infinity if not allowing to remove more than $z$ outliers (i.e., $\frac{1}{\epsilon}=\infty$). But our proposed coreset method in Theorem~\ref{the-coreset} can guarantee the clustering quality for excluding exactly $z$ outliers.

	\textbf{(2)} The coefficient ``$2$'' of $z$ in the coreset size actually can be further reduced by modifying the value of $\epsilon$ in Step 2 of Algorithm~\ref{alg-coreset} (we set $\epsilon=1$ just for simplicity). In general, the size of $E$ is 
	\begin{eqnarray*}
   (1+\epsilon)z+O\big(\frac{1}{\epsilon}(\frac{2}{\mu})^\rho (k+\ln\frac{1}{2\eta})\ln\frac{1}{2\eta}\big)
	\end{eqnarray*} 
	and  the construction time is $O(n\frac{1}{\epsilon}(\frac{2}{\mu})^\rho (k+\ln\frac{1}{2\eta})\ln\frac{1}{2\eta} )$. 	
\end{remark}

\subsection{When the Doubling Dimension $\rho$ Is Not Given} 

In Algorithm~\ref{alg-coreset}, we run Algorithm~\ref{alg-bi} $t=\frac{cl}{1-\eta}$ rounds. But when the doubling dimension $\rho$ is not given, we cannot determine the values of $l$ and $t$. We are aware of several techniques for estimating the doubling dimension of a given data set~\citep{har2006fast}. \cite{DBLP:journals/corr/abs-1802-09205} also mentioned that their coreset construction method can be applied to the case that even $\rho$ is not given. These ideas mainly rely on the fact that if one runs the Gonzalez's $k$-center clustering algorithm on the data, the obtained radius can be significantly reduced due to the property of doubling metrics. However, we need to emphasize that these doubling dimension estimation techniques cannot be applied to our problem under Assumption~\ref{ass-1}, since the outliers and inliers are mixed and only the inliers have the nice property of doubling metrics. 
We perform the following modification for Algorithm~\ref{alg-coreset}. \textbf{Roughly speaking, we decompose Step 2 of Algorithm~\ref{alg-coreset} into two substeps.}

\textbf{(1)} First, we run 
Algorithm \ref{alg-bi}  $\tilde{k}=\frac{ck}{1-\eta}$ rounds and then obtain the radius $\tilde{r}=\phi_{1}(E,X)\leq 2r_{\mathtt{opt}}$. Now $X$ is partitioned into $\tilde{k}$ clusters $H_1, H_2, \cdots, H_{\tilde{k}}$ with excluding $2z$ outliers.  Each $H_j\cap X_{\mathtt{opt}}$ has a constant doubling dimension $\rho$ (note that each $H_j$ may also contain some points from $X\setminus X_{\mathtt{opt}}$). Also, the size $\Big|\big(\cup^{\tilde{k}}_{j=1}H_j\big)\setminus X_{\mathtt{opt}}\Big|\leq z$, and it implies
\begin{eqnarray*}
\Big|\big(\cup^{\tilde{k}}_{j=1}H_j\big)\cap X_{\mathtt{opt}}\Big|=\Big| \cup^{\tilde{k}}_{j=1}H_j \Big|-\Big|\big(\cup^{\tilde{k}}_{j=1}H_j\big)\setminus X_{\mathtt{opt}}\Big|\geq n-2z-z= n-3z. 
\end{eqnarray*}
Therefore, if we view the instance $(X,\mathtt{d})$ as an instance of $\tilde{k}$-center clustering with $3z$ outliers,  the optimal radius (denote by $r^{(-3z)}_{\mathtt{opt}}$) should be at most $\tilde{r}$. Overall, we have the upper and lower bounds for $\tilde{r}$:
\begin{eqnarray*}
r^{(-3z)}_{\mathtt{opt}}\leq \tilde{r}\leq 2r_{\mathtt{opt}}.
\end{eqnarray*}

\textbf{(2)} Then, if we run Step~3 of  Algorithm~\ref{alg-bi} (replacing ``$z$'' by ``$3z$'') with at most $t=\frac{cl'}{1-\eta}$ rounds where
\begin{eqnarray*}
l'=\big(\frac{r^{(-3z)}_{\mathtt{opt}}}{\frac{1}{4}\mu \tilde{r}}\big)^\rho \tilde{k}\leq\big(\frac{r^{(-3z)}_{\mathtt{opt}}}{\frac{1}{4}\mu r^{(-3z)}_{\mathtt{opt}}}\big)^\rho \tilde{k}=O\big((\frac{4}{\mu})^\rho k\big),
\end{eqnarray*}
the obtained radius (excluding the farthest $6z$ vertices) should be at most 
$$2\times \frac{1}{4}\mu \tilde{r}\leq 2\times\frac{1}{2}\mu r_{\mathtt{opt}}=\mu r_{\mathtt{opt}}.$$ 
Then we can use the similar idea of the proof of Theorem~\ref{the-coreset} to show that the obtained set $E$ is a qualified $\mu$-coreset. 

Overall, we have Algorithm~\ref{alg-coreset2} for the case that the doubling dimension $\rho$ is not given. The time complexity is $O(\frac{cl'}{1-\eta} n)=O\big(n(\frac{4}{\mu})^\rho (k+\ln\frac{1}{2\eta})\ln\frac{1}{2\eta}\big)$, and the coreset size is $6z+O\big((\frac{4}{\mu})^\rho (k+\ln\frac{1}{2\eta})\ln\frac{1}{2\eta}\big)$.

\begin{algorithm}[tb]
   \caption{Coreset Construction in Doubling Metrics with Unknown $\rho$}
   \label{alg-coreset2}
\begin{algorithmic}
  \STATE {\bfseries Input:} An instance $(X,\mathtt{d})$ of metric $k$-center clustering with $z$ outliers, and $|X|=n$; parameters $\mu\in(0,1)$ and $\eta\in (0,1/2)$.
   \STATE
\begin{enumerate}

\item Set $\epsilon=1$ and run Algorithm~\ref{alg-bi} $t=\frac{ck}{1-\eta}$ rounds where $c=2+\frac{2}{k(1-\eta)}\ln\frac{1}{\eta}$.
Denote by $\tilde{r}=\phi_1(X,E)$ the maximum distance between $E$ and $X$ by excluding the farthest $2z$ vertices, after the final round of Algorithm~\ref{alg-bi}.

\item Continue to run Step 3 of  Algorithm~\ref{alg-bi} (but replacing ``$z$'' by ``$3z$'') until $\phi_5(X,E)\leq \frac{1}{2}\mu \tilde{r}$ (i.e., excluding $6z$ outliers).

\item  Set $\tilde{r}'=\phi_5(X,E)$. Let 
$X_{\tilde{r}'}=\{p\mid p\in X \text{ and } \mathtt{d}(p, E)\leq \tilde{r}'\}$.

\item For each vertex $p\in X_{\tilde{r}'}$, assign it to its nearest neighbor in $E$; for each vertex $q\in E$, let its weight be the number of vertices assigning to it.

\item Add $X\setminus X_{\tilde{r}'}$ to $E$;  each vertex of $X\setminus X_{\tilde{r}'}$ has weight $1$.

\end{enumerate}
  \STATE {\bfseries Output} $E$ as the coreset.
\end{algorithmic}
\end{algorithm}

\begin{theorem}
\label{the-coreset-unknown}
With  probability at least $1-2\eta$, Algorithm~\ref{alg-coreset2} outputs a $\mu$-coreset $E$ of $k$-center clustering with $z$ outliers. The size of $E$ is at most $6z+O\big((\frac{4}{\mu})^\rho (k+\ln\frac{1}{2\eta})\ln\frac{1}{2\eta}\big)$, and the construction time is $O((\frac{4}{\mu})^\rho (k+\ln\frac{1}{2\eta})n\ln\frac{1}{2\eta})$.
\end{theorem}

We can apply the same idea of Remark~\ref{rem-coreset}~(2) to reduce the coreset size to be $3(1+\epsilon)z+O\big(\frac{1}{\epsilon}(\frac{4}{\mu})^\rho (k+\ln\frac{1}{2\eta})\ln\frac{1}{2\eta}\big)$, and meanwhile, the time complexity becomes $O\big(\frac{n}{\epsilon}(\frac{4}{\mu})^\rho (k+\ln\frac{1}{2\eta})\ln\frac{1}{2\eta} \big)$.

\section{Coreset for Distributed Data}
\label{sec-distriubted}
In this section, we consider the coreset for distributed clustering in the coordinator model \citep{DURIS199890}. Suppose the data $X=\sqcup_{i=1}^sX_i$ are distributed disjointly among $s\geq 2$ sites; all the sites can communicate with a central server. 
 Let $O=\sqcup_{i=1}^sO_i$, where $O_i\subset X_i$, be the set of outliers in the optimal solution; also suppose each  $|O_i|=z_i^*$. Let $X\setminus O=\sqcup_{j=1}^kC_j$ be the $k$ optimal clusters.  Note that the value of each $z_i^*$ is unknown. Thus a straightforward approach is to compute a coreset for  the $k$-center clustering with $z$ outliers on each $X_i$, and directly send the obtained coresets to the central server.  
Let $B$ be the information encoding a point. Obviously this approach takes a communication cost 
\begin{eqnarray}
\Big(2sz+s\cdot  O\big((\frac{2}{\mu})^\rho(k+\log\frac{1}{2\eta})\ln\frac{1}{2\eta}\big)\Big)B,\label{for-comm-baseline}
\end{eqnarray}
which can be to too high if $z$ is large (e.g., if $z=5\%n$ and $s=10$, the cost can be larger than $nB$).

In this section, we show that the framework for distributed $k$-median/means clustering with outliers developed by \citet{guha2017distributed} can also be applied to the $k$-center clustering with $z$ outliers problem with our proposed coreset method in Section~\ref{sec-doubling}; in particular, the term ``$2sz$'' of (\ref{for-comm-baseline}) can be reduced to be ``$4z$''. The high level idea of \citet{guha2017distributed} is as follows. First, we need to design a set of numbers $\{z_1,z_2,\dots,z_s\}$, where each $z_i$ is an upper bound of $z_i^*$ for $1\leq i\leq s$ and their sum $\sum^s_{i=1}z_i\leq 2z$. Note that the requirement ``$\sum^s_{i=1}z_i\leq 2z$'' is important for bounding the total communication cost. Each site $i$ runs the coreset algorithm for $k$-center clustering with $z_i$ outliers on $X_i$ to construct a local coreset. Then each site sends the weighted points of the local coreset to the central server. Finally the central server aggregates the weighted points to form a global coreset. The key challenge is to compute the set $\{z_1,z_2,\dots,z_s\}$ that are suitable for our coreset method. 

Below we introduce some notations first. 
\begin{enumerate}
	\item  {\bf $r_{\mathtt{opt}}(X_i,k,z^*_i)$:} the optimal radius of  $k$-center clustering with $z^*_i$ outliers  on $X_i$.
	\item Given a set of $2$-dimensional points $A=\{(x_1,y_1),\dots,(x_l,y_l)\}\subset\mathbb{R}^2$ where $x_1<x_2<\dots<x_l$, we define the corresponding piecewise function  $h_A(\cdot)$ from $[x_1,\infty)$ to $\mathbb{R}$: $h_A(x)=y_i$ if $x_i\leq x< x_{i+1}$. Here we define $x_{l+1}=\infty$. See Figure~\ref{fig-piecewise} for an illustration.

\end{enumerate}
\begin{figure}[tbp]
\begin{center}
    \includegraphics[width=0.55\textwidth]{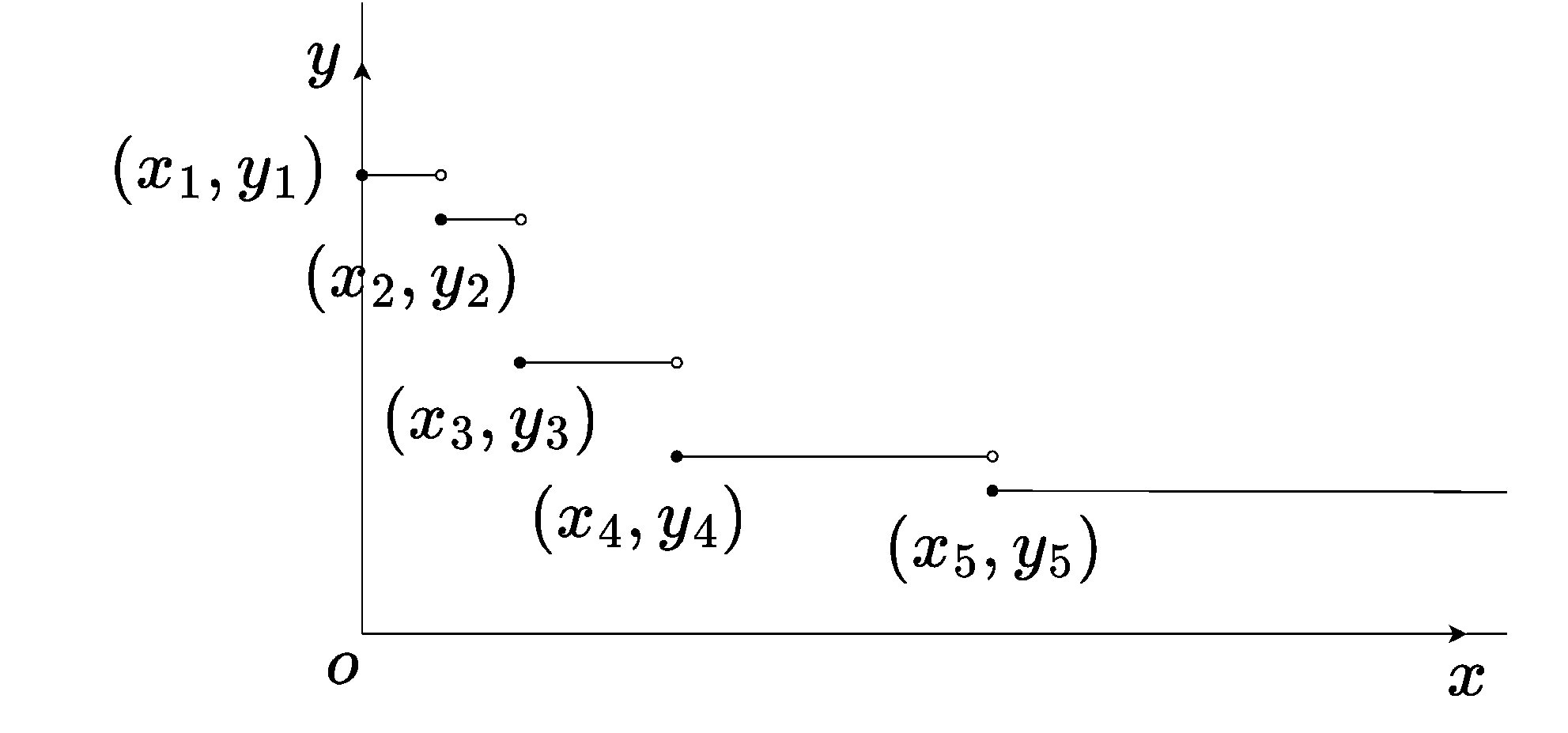}  
    \end{center}
\captionsetup{format=hang}
  \caption{An illustration for  a non-increasing  piecewise function $h_A(\cdot)$ with $A=\{(x_1,y_1),(x_2,y_2),(x_3,y_3),(x_4,y_4),(x_5,y_5)\}$.}   
   \label{fig-piecewise}
\end{figure}

For any pair of $(x_i, y_i)$ and $(x_j, y_j)$, we define their \textbf{lexicographical order:}
\begin{equation}
\left(x_i, y_i\right) \prec\left(x_j, y_j\right) \quad \text { if }\left\{\begin{array}{l}
x_i<x_j ; \text { or } \\\nonumber
x_i=x_j \text { and } y_i<y_j.
\end{array}\right.
\end{equation}

\begin{algorithm}[tb]
   \caption{Distributed Coreset Construction}
   \label{alg-distributed_coreset}
\begin{algorithmic}
  \STATE {\bfseries Input:} An instance $(X,d)$ of distributed metric $k$-center clustering with $z$ outliers, and $X=\sqcup_{i=1}^s X_i$; the  parameters  $\eta\in(0,1/2),\mu\in (0,1)$. 
   \STATE
\begin{enumerate}
\item  Let $[z]=\{0,1, 2, \cdots, z\}$ and  $\Gamma=\left\{2^{r}: 1 \leq r \leq\left\lfloor\log _{2} z\right\rfloor, r \in \mathbb{Z}\right\} \cup\{0, z\}$.  Run the following two-round communication  between the sites and the central server. 

\item \textbf{(1st round) In each site $i$:} run Algorithm \ref{alg-coreset} to obtain the radius $\tilde r_{i,q}$ and the coreset $E_{i,q}$ for each $q\in\Gamma$, where $\tilde r_{i,q}$ is the radius ``$\tilde r$'' in Algorithm \ref{alg-coreset} with setting the number of outliers to be $q$.
\begin{enumerate}
\item Define the set $A_i=\{(q,\tilde r_{i,q})\mid q\in\Gamma\}$, and construct the corresponding piecewise function $h_{A_i}$;
\item Send the function  $h_{A_i}$ to the central server. 
\end{enumerate}

\item \textbf{(1st round) In the central server:} sort the $s(z+1)$ pairs $\{( h_{A_i}(q),i)\mid i=1, 2, \cdots, s; q\in [z]\}$ with a lexicographical decreasing order; select the $(2z+1)$-th largest item, say ``$(h_{A_{i_0}}(q_0),i_0)$'', and broadcast it to all the sites.

\item \textbf{(2nd round) In each site $i$:} \begin{enumerate}

\item If $i\not=i_0$, let $z_i=\min\{q\in [z]\colon \left(h_{A_i}(q),i\right)\prec (h_{A_{i_0}}(q_0),i_0)\}$ (if the set is $\emptyset$, let  $z_i= z$);
\item Else, $i=i_0$, let $z_{i_0}= \min\{q\in\Gamma\colon h_{A_{i_0}}(q)= h_{A_{i_0}}(q_0)\}$;

\item Send $E_{i,z_i}$ to the central server.

\end{enumerate}  

\item \textbf{(2nd round) In the central server:} take the union $E=\cup^s_{i=1}E_{i,z_i}$ as the final coreset.

\end{enumerate}

\end{algorithmic}
\end{algorithm}

\begin{theorem}
\label{the-distributed}

	With probability at least  $1-2s(2+\log_2 z)\eta$, Algorithm \ref{alg-distributed_coreset} returns a $2\mu$-coreset $E$ of $k$-center clustering with $z$ outliers for the distributed input $X=\sqcup_{i=1}^sX_i$. The total communication complexity is  $\Big(4z+O((\frac{2}{\mu})^{\rho}s(k+\ln\frac{1}{2\eta})\ln\frac{1}{2\eta})\Big)B$  over $2$ rounds, where $B$ is the communication cost for sending one point. The running time in each site is  $O((\frac{2}{\mu})^{\rho}(k+\ln\frac{1}{2\eta})n_i\ln\frac{1}{2\eta}\log_{2}z)$.
\end{theorem}

\begin{remark}

 We can replace $\eta$ by $\frac{\eta}{2s(2+\log_2 z)}$ in Algorithm \ref{alg-distributed_coreset} to achieve a success probability   $1-\eta$. The communication complexity will be $$\Big(4z+ O\big((\frac{2}{\mu})^{\rho}s(k+\ln\frac{s\log_2 z}{\eta})(\ln\frac{s \log_2 z}{\eta})\big)\Big) B.$$ 
\end{remark}

Before proving Theorem~\ref{the-distributed}, we introduce Lemma~\ref{cl-distributed} and Lemma~\ref{lem-guess} first.

\begin{lemma}
\label{cl-distributed}
	$2r_{\mathtt{opt}}\geq \max_{1\leq i\leq s}r_{\mathtt{opt}}(X_i,k,z_i^*)$.
\end{lemma}
Note that though $X_i\subset X$, the optimal radius $r_{\mathtt{opt}}(X_i,k,z_i^*)$ of site $i$ is not necessary to be $\leq r_{\mathtt{opt}}$, since the cluster centers of $X$ may not belong to $X_i$ (but for the problem in Euclidean space, $r_{\mathtt{opt}}(X_i,k,z_i^*)$ is always no larger than $ r_{\mathtt{opt}}$ since the cluster centers can be any points in the space). 
$X_i\setminus O_i$ is the set of inliers of site $i$. For each optimal cluster $C_j$, $1\leq j\leq k$, we can arbitrarily take an inlier from 
$(X_i\setminus O_i)\cap C_j$ as the surrogate cluster center (if $(X_i\setminus O_i)\cap C_j=\emptyset$, we just ignore this cluster). From the triangle inequality, we know $r_{\mathtt{opt}}(X_i,k,z_i^*)\leq 2r_{\mathtt{opt}}$ and thus obtain  Lemma~\ref{cl-distributed}.

In the following analysis, we assume that the function $h_{A_i}$ in Step 2(a) is \textbf{non-increasing} for each $i=1, 2, \cdots, s$. Actually this assumption is easy to satisfy. If there exist a couple $q<q'$ such that  $\tilde r_{i,q}>\tilde r_{i,q'}$, we can simply replace the coreset $E_{i,q}$ by the coreset $E_{i,q'}$ and let $\tilde r_{i,q}=\tilde r_{i,q'}$ in Step 2 of Algorithm \ref{alg-distributed_coreset}.  
The following lemma illustrates the key properties of the obtained values $z_1, z_2, \cdots, z_s$ in Algorithm \ref{alg-distributed_coreset}.

\begin{lemma}
\label{lem-guess}
	The  set $\{z_1,\dots,z_s\}$ obtained in Step 4 of Algorithm~\ref{alg-distributed_coreset} is the optimal solution for  the following minimax problem:
	
\begin{equation}
\label{eq:minimax}
\begin{array}{cll}
&\min\limits_{q_1,\dots,q_s}\max\limits_{1\leq i\leq s} h_{A_i}(q_i) & \\
\text {s.t. } & \sum\limits_{i=1}^s q_i\leq 2 z, \\
& q_i\in [z],~~ i=1,\dots,s. & \\
\end{array}
\end{equation}
\end{lemma}
\begin{proof}
	 We consider the following two cases. Recall that ``$i_0$'' is the index obtained in Step~3. 
	
	\textbf{Case $(\romannumeral1)$: } $h_{A_{i_0}}(z_{i_0})=\max_i h_{A_i}(z_i)$.  In this case, by the definition of $z_i$ and the fact that $h_{A_i}(\cdot)$ is non-increasing, we have $(h_{A_{i_0}}(q_0),i_0)\prec (h_{A_i}(q),i)$ for each $q=0,\dots,z_i-1$ and each  $i=1,\dots, s$. Hence there are $\sum_{i=1}^s z_i$ pairs of $(h_{A_i}(q),i)$s that are larger than $(h_{A_{i_0}}(q_0),i_0)$ in the lexicographical order . Therefore by the definition of $(h_{A_{i_0}}(q_0),i_0)$,  we know that $z_{i_0}\leq q_0$, which implies 
 \begin{eqnarray*}
 \sum_{i=1}^s z_i\leq q_0+\sum_{i\neq i_0}z_i. 
 \end{eqnarray*}
 The right-hand side ``$q_0+\sum_{i\neq i_0}z_i$'' is exactly equal to $2z$ since it is the number of items ranked ahead of $(h_{A_{i_0}}\left(q_0\right), i_0)$ in the sorted sequence in Step 3 of Algorithm \ref{alg-distributed_coreset}. Suppose Lemma~\ref{lem-guess} is not true, then there exists another solution $\{z'_1, \cdots, z'_s\}$ of the problem  (\ref{eq:minimax}) such that 
	\begin{eqnarray}
	\label{eq:contradiction}
	\max_i h_{A_i}(z'_i)< h_{A_{i_0}}(z_{i_0}).
	\end{eqnarray}
	The right-hand side  ``$h_{A_{i_0}}(z_{i_0})$'' is equal to $h_{A_{i_0}}(q_0)$ by the definition of $z_{i_0}$. So it implies $h_{A_{i_0}}(z'_{i_0})<h_{A_{i_0}}(q_0)$; since $h_{A_{i_0}}(\cdot)$ is non-increasing, we know $z'_{i_0}>q_0$. Without loss of generality, we assume $\sum_iz'_i=2z$ (again, because $h_{A_i}(\cdot)$ is non-increasing, we can always enlarge the $z'_i$s until $\sum_iz'_i=2z$). Note that $q_0+\sum_{i\neq i_0}z_i= 2 z$.  So there should exist an index $ j\neq i_0$, such that $z'_j<z_j$. By the definition of $z_j$, we have $(h_{A_{i_0}}(q_0),i_0)\prec(h_{A_j}(z'_j), j)$. Therefore $h_{A_j}(z'_j)\geq h_{A_{i_0}}(q_0) = h_{A_{i_0}}(z_{i_0})$.  Thus $\max_i h_{A_i}(z'_i)\geq h_{A_j}(z_j')\geq h_{A_{i_0}}(z_{i_0})$, which is contradictory to (\ref{eq:contradiction}).
	
	\textbf{Case $(\romannumeral2)$:} suppose $h_{A_{i_1}}(z_{i_1})=\max_i h_{A_i}(z_i)$  and $h_{A_{i_1}}(z_{i_1})> h_{A_{i_0}}(z_{i_0})$. In this case we have $z_{i_1}=z$ in Step 4(a).
	Similar to the analysis for case $(\romannumeral1)$, we have $(h_{A_{i_0}}(q_0),i_0)\prec (h_{A_i}(q),i)$ for $q=0,\dots,z_i-1$, $i\neq i_0$, and meanwhile  $(h_{A_{i_0}}(q_0),i_0)\prec (h_{A_{i_1}}(q),i_1)$ for $q=0,\dots,z_{i_1}$. Hence similarly we have $1+\sum_{i=1}^sz_i\leq q_0+1+\sum_{i\neq i_0} z_i=2z$. For any feasible solution $\{z'_1,\dots,z'_s\}$, since $z'_{i_1}\leq z=z_{i_1}$ and $h_{A_{i_1}}(\cdot)$ is non-increasing, we have 
 \begin{eqnarray*}
 \max_i h_{A_i}(z'_i)\geq h_{A_{i_1}}(z'_{i_1})\geq h_{A_{i_1}}(z_{i_1})=\max_i h_{A_i}(z_i), 
 \end{eqnarray*}
 which implies $\{z_1,\dots,z_s\}$ is better than the solution $\{z'_1,\dots,z'_s\}$. So $\{z_1,\dots,z_s\}$ should be the optimal solution of the problem (\ref{eq:minimax}).
\end{proof}

\begin{proof}\textbf{(of Theorem~\ref{the-distributed})}
	We define $\hat z_i=\min\{q\in \Gamma\colon q\geq z_i^*\}$ for $i=1,2, \cdots, s$. 
	 Then we directly have $\max_i ~r_{\mathtt{opt}}(X_i,k,z_i^*)\geq \max_i ~r_{\mathtt{opt}}(X_i,k,\hat z_i)$ since $\hat z_i\geq z_i^*$. 
  Because ``$r_{i,\hat z_i}$'' is the radius of the coreset in Step 2, we have 
	\begin{eqnarray*}
	\tilde r_{i,\hat z_i}\leq\mu r_{\mathtt{opt}}(X_i,k,\hat z_i).
	\end{eqnarray*}
	Also we   know $2r_{\mathtt{opt}} \geq \max_i ~r_{\mathtt{opt}}(X_i,k,z_i^*)$ from  Lemma \ref{cl-distributed}.	So we have  $2r_{\mathtt{opt}} \geq \frac{1}{\mu}\max_i \tilde r_{i,\hat z_i}$. Note that $h_{A_i}(\hat z_i)=\tilde r_{i,\hat z_i}$ as $\hat z_i\in\Gamma$. Hence we have 
	\begin{eqnarray}
	\label{eq:dis-1}
	2r_{\mathtt{opt}} \geq \frac{1}{\mu} \max_i h_{A_i}(\hat z_i).
	\end{eqnarray}
	
	The definitions of $\hat z_i$ and $\Gamma$ together imply that $\sum_i\hat z_i\leq \sum_i 2z^*_i=2z$. Therefore the set $\{\hat z_1,\dots,\hat z_s\}$ is a feasible solution for the problem (\ref{eq:minimax}). By Lemma \ref{lem-guess}, we have
	
	\begin{eqnarray}
	\label{eq:dis-2}
	\max_i h_{A_i}(\hat z_i)\geq \max_i h_{A_i}(z_i).
	\end{eqnarray}
	
	For each $i\not=i_0$, we know $z_i\in\Gamma$ by the definition of the piecewise function $h_{A_i}(\cdot)$. By Step 4(b) of Algorithm~\ref{alg-distributed_coreset}, we know $z_{i_0}\in\Gamma$. Thus we have 
	\begin{eqnarray}
	\label{eq:dis-3}
	h_{A_i}(z_i)=\tilde r_{i,z_i}=\phi_1(X_i,E_{i,z_i}), ~i=1,\dots,s.
	\end{eqnarray}
	
	Combining the inequalities  (\ref{eq:dis-1}), (\ref{eq:dis-2}),  and  (\ref{eq:dis-3}), we have 
	\begin{eqnarray*}
	\max_i \phi_1(X_i,E_{i,z_i})\leq 2\mu \cdot r_{\mathtt{opt}}.
	\end{eqnarray*}

Similarly to the inequality (\ref{for-map}) in the proof of Theorem \ref{the-coreset}, we can define the  bijective mapping ``$f$''  from $X$ to $E$ (recall that $E=\cup_i^s E_{i,z_i}$), such that
\begin{align}
	\mathtt{d}(p,f(p))\leq 2\mu \cdot r_{\mathtt{opt}}, ~\forall p\in X.\label{ineq-distributed-coreset}
\end{align}
The above inequality (\ref{ineq-distributed-coreset}) implies that the set $E$ is a qualified $2\mu$-coreset.

The success probability is at least $1-s(2+\log_2z)2\eta$ since each site runs  Algorithm~\ref{alg-coreset} no more than $(2+\log_2 z)$ times.  Since $\sum_i^sz_i\leq2z$, the total number of points sent from the sites to the central server is no larger than 
  $4z+ O((\frac{2}{\mu})^{\rho}s(k+\ln\frac{1}{2\eta})\ln\frac{1}{2\eta})$. So we obtain the communication complexity $\Big(4z+O((\frac{2}{\mu})^{\rho}s(k+\ln\frac{1}{2\eta})\ln\frac{1}{2\eta})\Big)B$. The running time in each site $i$ is $ O(|\Gamma|(\frac{2}{\mu})^{\rho}(k+\ln\frac{1}{2\eta})n_i\ln\frac{1}{2\eta})=O((\frac{2}{\mu})^{\rho}(k+\ln\frac{1}{2\eta})n_i\ln\frac{1}{2\eta}\log_{2}z)$. 
\end{proof}

\section{Experiments}
\label{sec-exp}
All the experiments were conducted on an Ubuntu workstation with 2.40GHz Intel(R) Xeon(R) CPU E5-2680 and 256GB main memory. The algorithms were implemented in MATLAB R2019b. Our code is available at \url{https://github.com/OpsTreadstone/randomized-k-center}.

\textbf{Baselines.} We compare our algorithms with two well known baselines, ``\textsc{CKM+}''~\citep{charikar2001algorithms} and ``\textsc{MK}''~\citep{mccutchen2008streaming}, as well as the recently proposed algorithm  ``\textsc{BVX}''~\citep{NEURIPS2019_73983c01}. 
For the coreset construction problem, we compare our algorithm with ``\textsc{CPP}''~\citep{DBLP:journals/corr/abs-1802-09205} and the uniform sampling method  ``\textsc{Uniform}''.

For the  distributed setting, we take  ``\textsc{CPP}'', ``\textsc{MKC+}''~\citep{malkomes2015fast}, ``\textsc{GLZ}''~\citep{guha2017distributed}, and ``\textsc{LG}''~\citep{li2018distributed} as the baselines. 

All the experiments were repeated 10 times and we report the average results with the standard deviations.

\textbf{Data sets.} We evaluate our algorithms on four real-world classification data sets from the UCI KDD archive~\citep{Dua:2019}: Shuttle, Covertype, KDD Cup 1999 and Poker Hand.
The Shuttle data set~\citep{king1995statlog} contains $43,500$ instances of 7 classes with 9 numerical attributes. 
The Covertype data set contains $581,012$ instances of 7 classes. It has 54 attributes of continuous and categorical types.
The KDD Cup 1999 data set contains $4,898,431$ instances of 23 classes with 41 attributes.
The Poker Hand data set contains $1,025,010$ instances of 10 classes with 10 attributes.
For each of the latter three data sets Covertype, KDD Cup 1999, and Poker Hand, we randomly select $100,000$ instances  and run the algorithms  on the selected instances.

To generate the outliers, for each data set  we compute the minimum enclosing ball of the whole data set by using the algorithm of \citet{badoiu2003smaller}; let $r_{\mathtt{meb}}$ and $c_{\mathtt{meb}}$ be the radius  and the center, respectively. Then  we randomly add $1\%$ points as the outliers inside the ball of radius $1.1 \times r_{\mathtt{meb}}$ centered at $c_{\mathtt{meb}}$.

\begin{figure}[tbp]
	\centering
	\includegraphics[width=\textwidth]{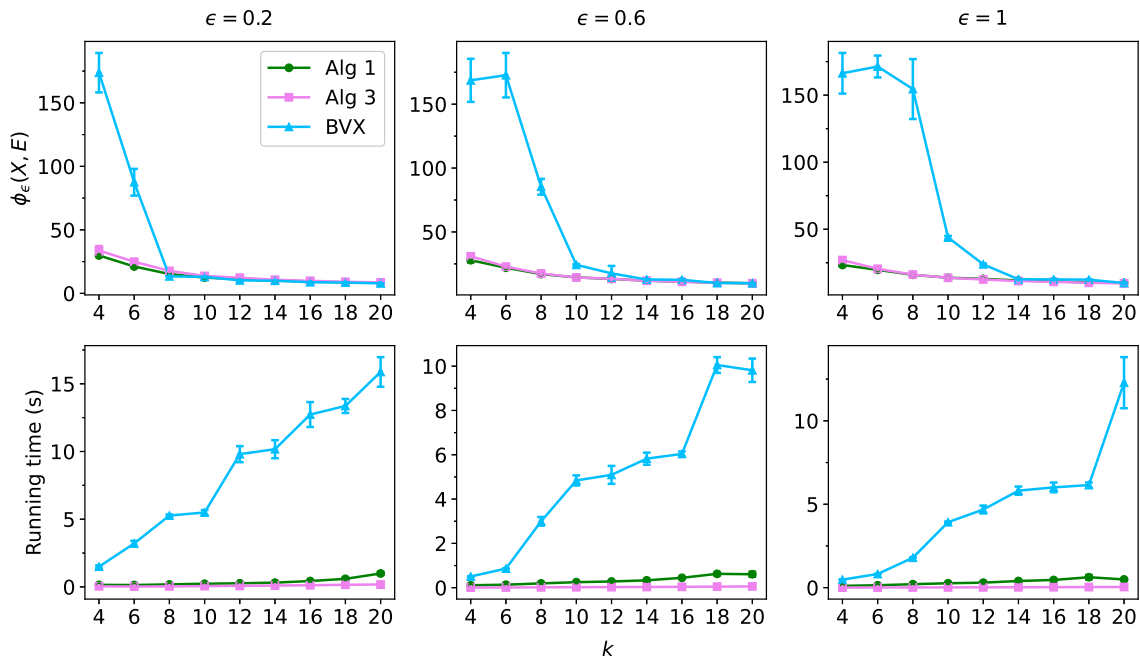}
	\caption{The performance of Algorithm~\ref{alg-bi} and Algorithm~\ref{alg-sub-bi} on Shuttle. 
 }
	\label{fig-alg1_alg3-shuttle}
\end{figure}

\begin{figure}[tbp]
	\centering
	\includegraphics[width=\textwidth]{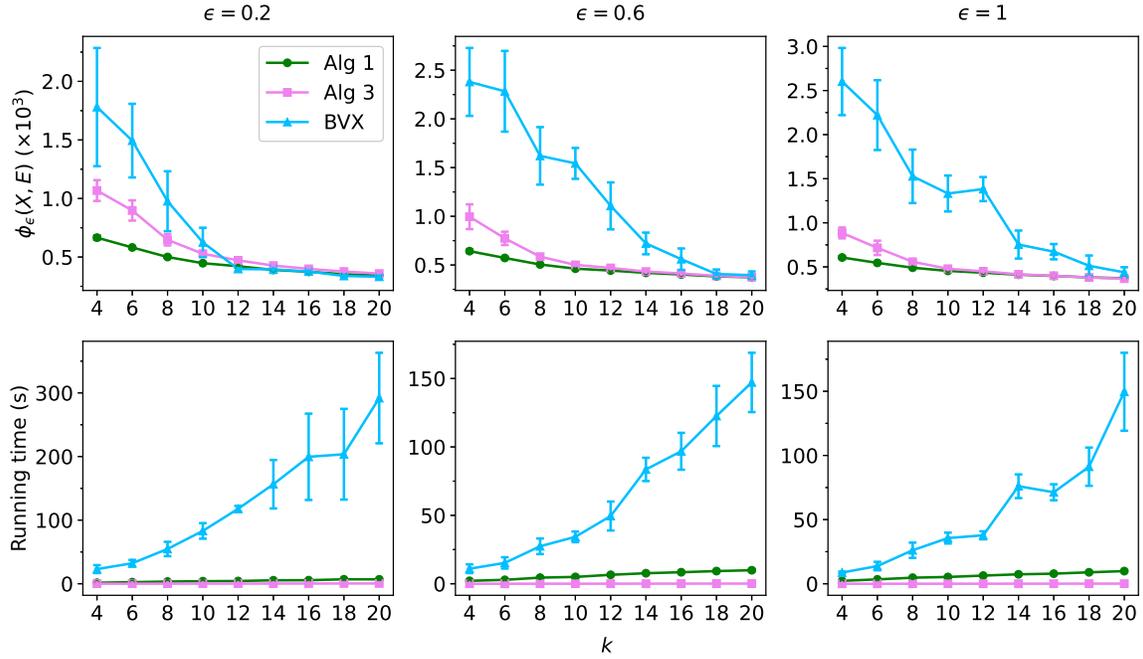}
	\caption{The performance of Algorithm~\ref{alg-bi} and Algorithm~\ref{alg-sub-bi} on Covertype.
 }
	\label{fig-alg1_alg3-tiny_covertype}
\end{figure}

\begin{figure}[tbp]
	\centering
	\includegraphics[width=\textwidth]{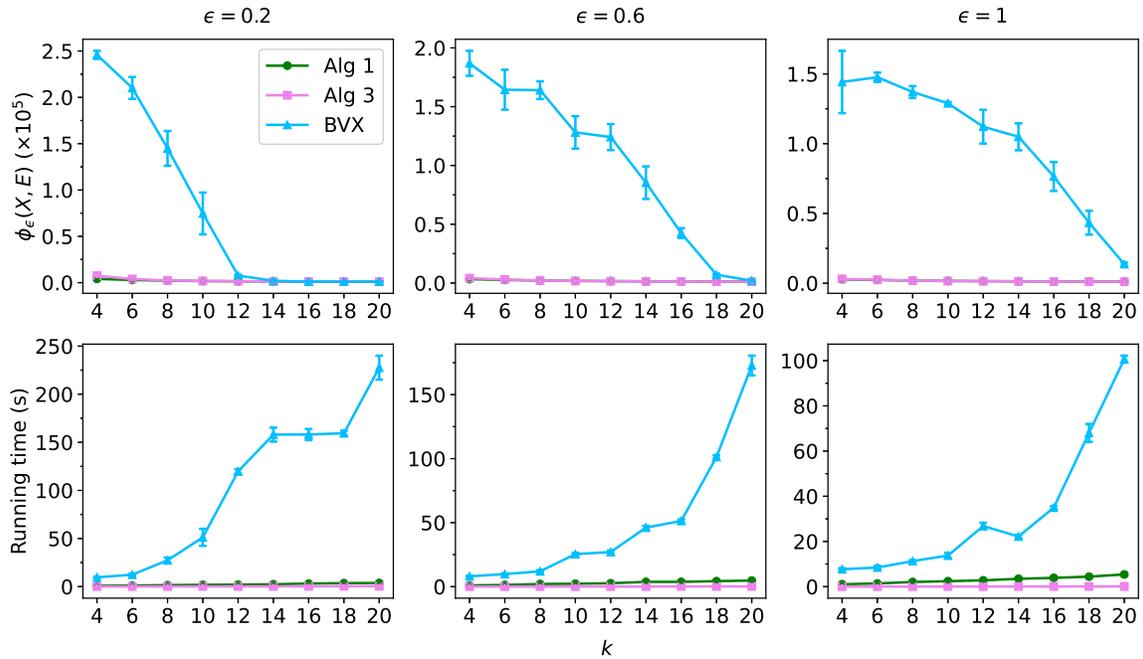}
	\caption{The performance of Algorithm~\ref{alg-bi} and Algorithm~\ref{alg-sub-bi} on KDD Cup 1999.
 }
	\label{fig-alg1_alg3-tiny_kddcup99}
\end{figure}

\begin{figure}[tbp]
	\centering
	\includegraphics[width=\textwidth]{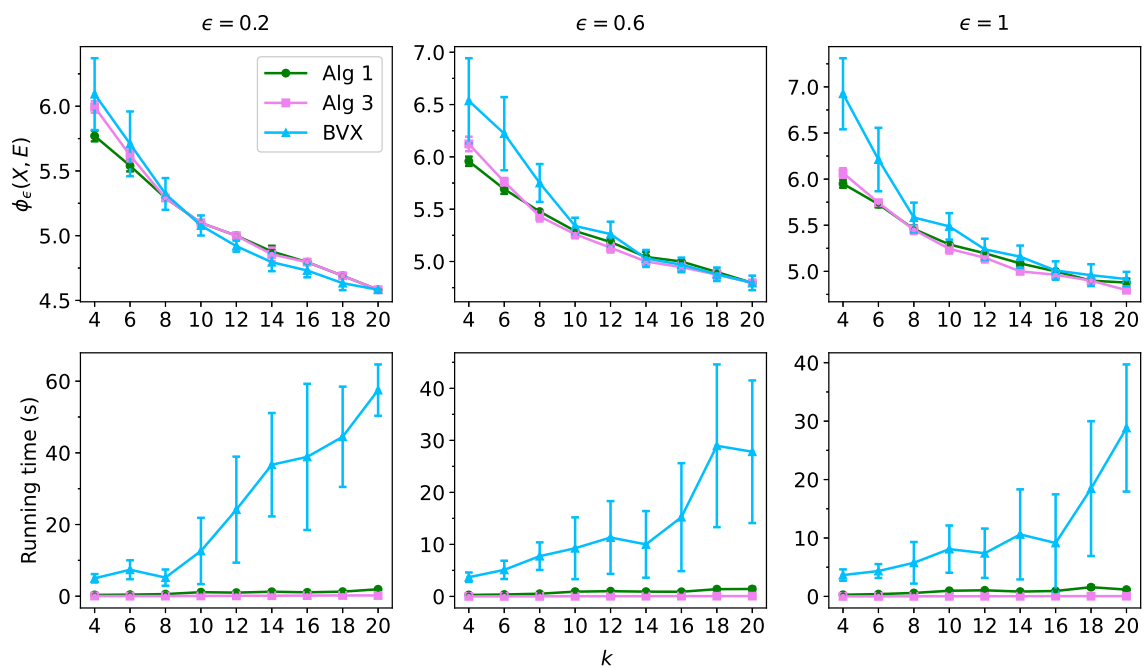}
	\caption{The performance of Algorithm~\ref{alg-bi} and Algorithm~\ref{alg-sub-bi} on Poker Hand. 
 }
	\label{fig-alg1_alg3-tiny_pokerhand}
\end{figure}

\begin{figure}[tbp]
	\centering
	\includegraphics[width=\textwidth]{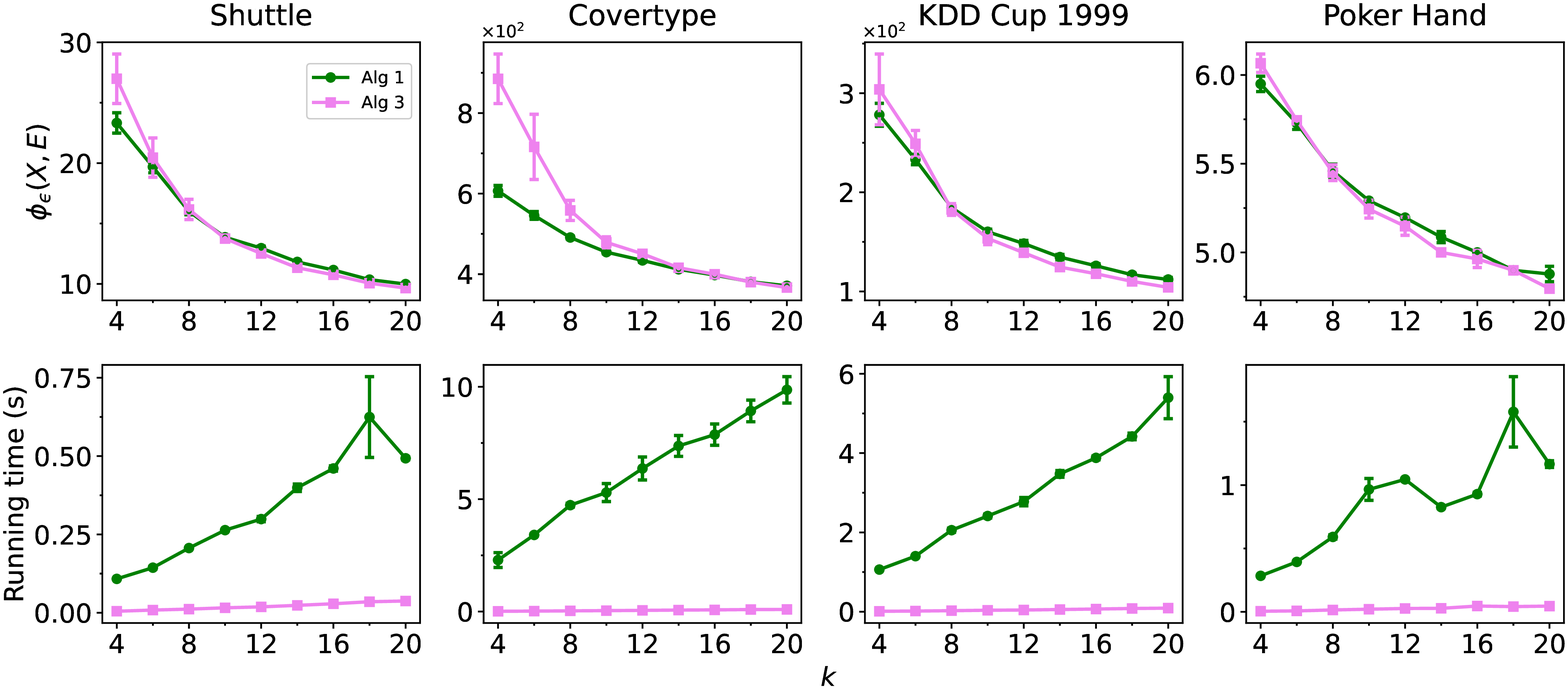}
	\caption{The comparison between Algorithm~\ref{alg-bi} and Algorithm~\ref{alg-sub-bi}.
 }
	\label{fig-alg1_alg3-only}
\end{figure}

\subsection{The Bi-criteria Algorithms}
\label{sec-expbi}
We compare Algorithm~\ref{alg-bi} and its sublinear version Algorithm~\ref{alg-sub-bi} with \textsc{BVX}. For Algorithm~\ref{alg-bi}, we set $\epsilon=0.2,0.6,1$, and modify the parameters of Algorithm~\ref{alg-sub-bi} and \textsc{BVX} accordingly so that they can output the same number of centers. We vary $k$ from $4$ to $20$. 
The experimental results are shown in Figure~\ref{fig-alg1_alg3-shuttle}, Figure~\ref{fig-alg1_alg3-tiny_covertype}, Figure~\ref{fig-alg1_alg3-tiny_kddcup99}, and Figure~\ref{fig-alg1_alg3-tiny_pokerhand}.
Comparing with \textsc{BVX}, Algorithm~\ref{alg-bi} and Algorithm~\ref{alg-sub-bi} take significantly lower running time, and meanwhile achieve similar or lower clustering cost $\phi_\epsilon (X,E)$. 

To have a more clear  comparison between Algorithm~\ref{alg-bi} and Algorithm~\ref{alg-sub-bi}, we zoom in on the experimental results of 
$\epsilon=1$ without \textsc{BVX} (see  Figure~\ref{fig-alg1_alg3-only}).
We can see that the running time of Algorithm~\ref{alg-sub-bi} grows much slower than Algorithm~\ref{alg-bi} as $k$ increases. This result also agrees with our theoretical analysis  since Algorithm~\ref{alg-sub-bi} has only sublinear time complexity.

We also compare Algorithm~\ref{alg-single} with \textsc{CKM+}, \textsc{MK}, and \textsc{BVX} for small $k$. We let $k=2$, $3$, $4$, $5$. 
For Algorithm~\ref{alg-single}, we set $\epsilon=1$ and run it $\frac{\ln 10}{1-\gamma}(\frac{1+\epsilon}{\epsilon})^{k-1}$ times as Corollary~\ref{the-kcenter2} suggests.
The experimental results  are shown in Figure~\ref{fig-alg2}. 
In general, Algorithm~\ref{alg-single} achieves comparable clustering cost with \textsc{CKM+} and \textsc{MK}, but runs faster than these two baselines. 
\textsc{BVX} is faster but has worse clustering cost than 
Algorithm~\ref{alg-single}.

\begin{figure}[!ht]
	\centering
	\includegraphics[width=\textwidth]{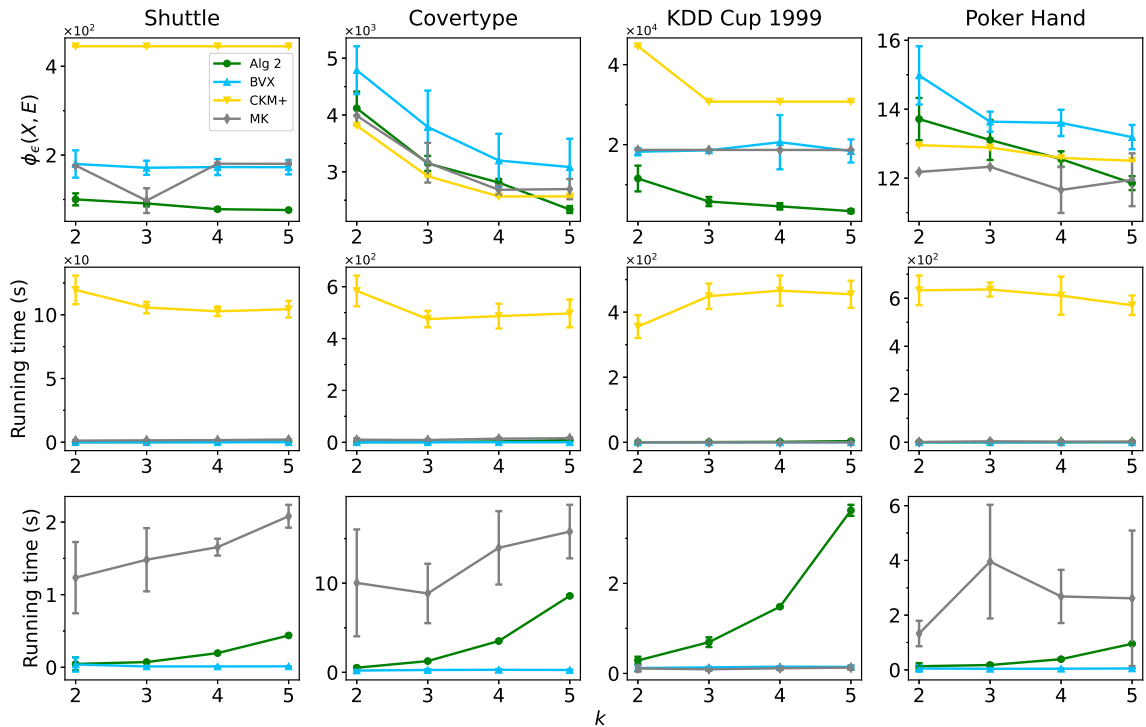}
    \captionsetup{format=hang}
	\caption{The performance of Algorithm~\ref{alg-single}. 
 The third row removes \textsc{CKM+} to have a more clear illustration on the running times of the other three algorithms.}
	\label{fig-alg2}
\end{figure}

\subsection{The Coreset Algorithms }
\label{sec-expcoreset}

We compare Algorithm~\ref{alg-coreset2} with the coreset methods \textsc{CPP}  and \textsc{Uniform}. We set the sizes of coreset to be $\{4\%n,8\%n,12\%n,16\%n,20\%n\}$ for these three methods, where $n$ is  the number of points (including the outliers). 
We run the algorithm \textsc{Cluster} proposed by \citet{malkomes2015fast}, which is a modification of \textsc{CKM+}, as the ``host'' algorithm on the obtained coresets constructed by Algorithm~\ref{alg-coreset2} and \textsc{Uniform}.  
 We let $\text{RT}_{\mathtt{coreset}}$ denote the   coreset construction time,  and let $\text{RT}_{\mathtt{total}}$ denote the total running time (including the coreset construction time and the time for running the $k$-center with outliers algorithm on the coreset). 
 To study the advantage of coreset, we also  
  compare with \textsc{CKM+} and \textsc{MK};  we directly run these two algorithms on the whole data sets (without coreset) to  compute the clustering results.  

 The experimental results are shown in Figure~\ref{fig-alg5}. Note that we illustrate the clustering cost $\phi_0 (X, E)$ (not $\phi_\epsilon (X, E)$) in the first row of Figure~\ref{fig-alg5} (and also Figure~\ref{fig-alg6} in Section~\ref{sec-expdistributed}), that is, we discard exactly $z$ outliers rather than $(1+\epsilon)z$. 
\textsc{Uniform} is always the fastest coreset method since it is only simple uniform sampling and does not need any construction procedure; but its clustering cost is worse than Algorithm~\ref{alg-coreset2} and \textsc{CPP} for most cases. Both of Algorithm~\ref{alg-coreset2} and \textsc{CPP} achieve lower clustering cost than \textsc{CKM+} and \textsc{MK}. 
Comparing with \textsc{CPP}, Algorithm~\ref{alg-coreset2} has lower clustering cost on Covertype and Poker Hand; Algorithm~\ref{alg-coreset2} also   has lower $\text{RT}_{\mathtt{coreset}}$ and $\text{RT}_{\mathtt{total}}$.
The experimental results suggest that Algorithm~\ref{alg-coreset2} can yield significant reduction on the running time (if setting the coreset size $\leq 12\%$) and achieve good clustering quality as well.

\begin{figure}[tbp]
	\centering
	\includegraphics[width=\textwidth]{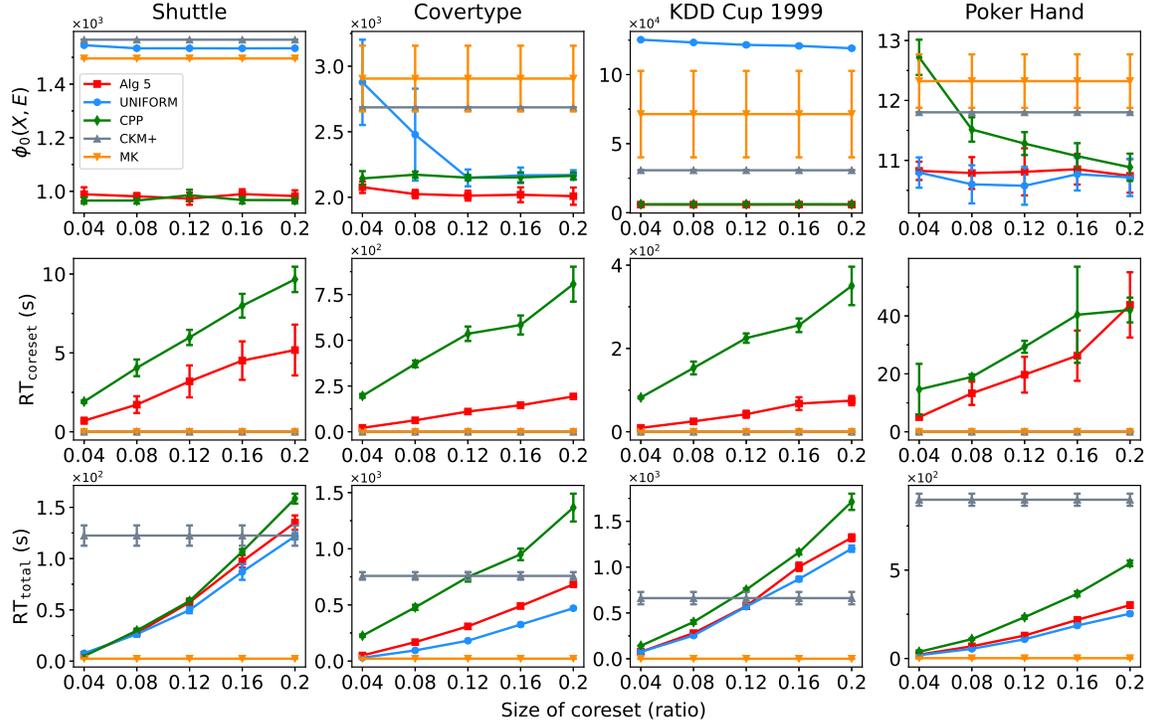}
	\caption{The performance of the coreset method Algorithm~\ref{alg-coreset2}.}
	\label{fig-alg5}
\end{figure}

\begin{figure}[tbp]
	\centering
	\includegraphics[width=\textwidth]{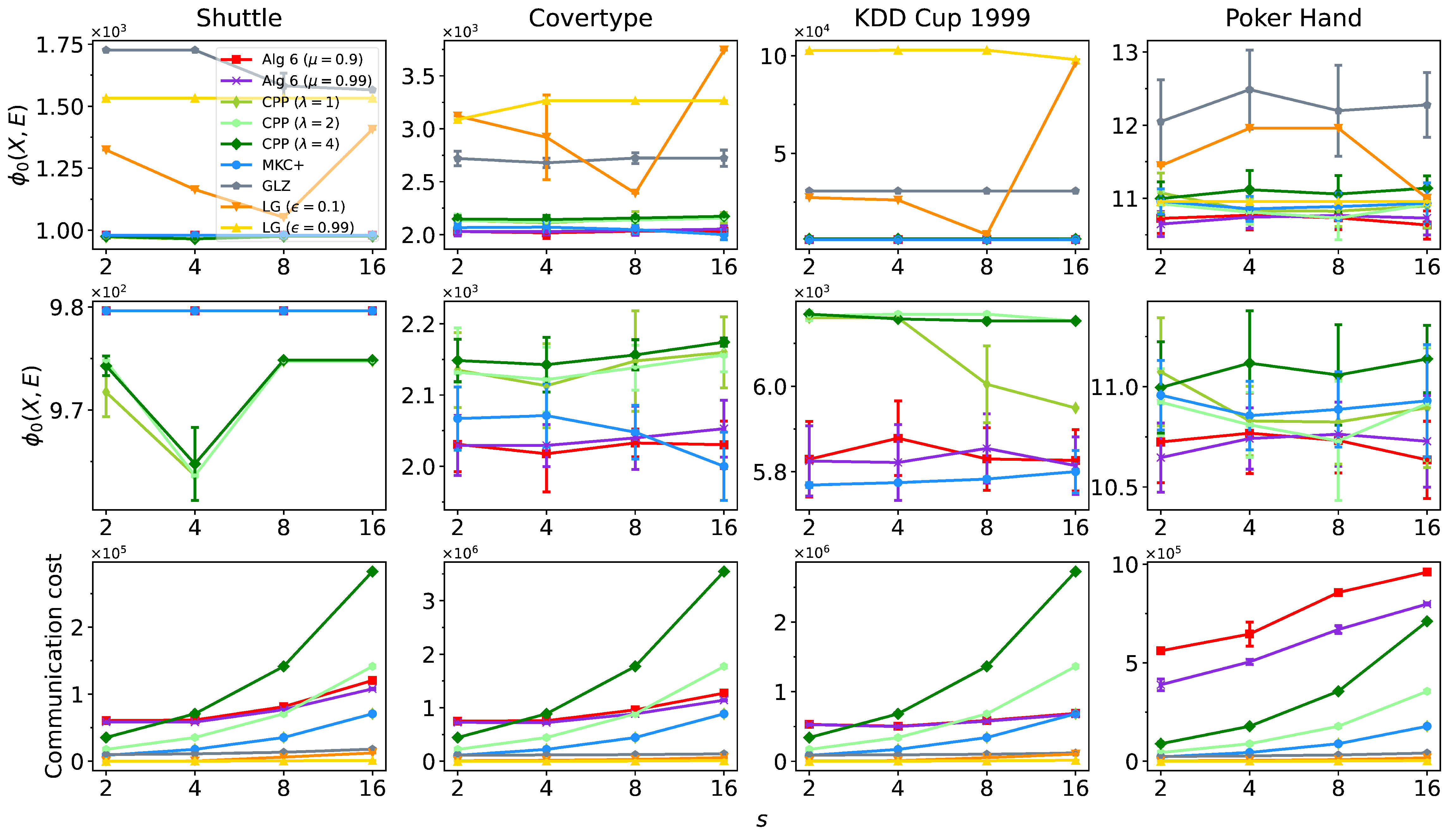}
    \captionsetup{format=hang}
	\caption{The performance of Algorithm~\ref{alg-distributed_coreset}. In the second row, we remove \textsc{GLZ} and \textsc{LG} (since they have much higher clustering costs than the others) and   zoom in on the comparison of other algorithms. 
 }
	\label{fig-alg6}
\end{figure}

\subsection{The Distributed Algorithm}
\label{sec-expdistributed}

We compare Algorithm~\ref{alg-distributed_coreset} with \textsc{CPP}, \textsc{MKC+}, \textsc{GLZ} and \textsc{LG} with varying the number of sites $s$. 
For Algorithm~\ref{alg-distributed_coreset}, in Step 2 we run Algorithm~\ref{alg-coreset2} instead of Algorithm~\ref{alg-coreset} since the doubling dimensions of the four data sets are unknown.  
Similar with Section~\ref{sec-expcoreset}, we also run the \textsc{Cluster} algorithm  on the coresets constructed by Algorithm~\ref{alg-distributed_coreset}. 
For \textsc{CPP}, following the setting of \citet{DBLP:journals/corr/abs-1802-09205}, each site sends a coreset of size $\lambda(k+z)$ to the central server with  $\lambda=1,2,4$. 
\textsc{LG} returns a $(k,z)_\epsilon$-center solution  and we set $\epsilon=0.1,0.99$ in the algorithm as suggested in their paper \citep{li2018distributed}.

The experimental results of clustering cost and communication cost on the four data sets are shown in Figure~\ref{fig-alg6}. 
The communication cost is measured by the total number of floating numbers sent between the sites and the central server. \textsc{GLZ} and \textsc{LG} have  lower communication costs, but yield much higher clustering costs. 
Algorithm~\ref{alg-distributed_coreset} can achieve quite low clustering cost, but takes higher communication cost comparing with \textsc{GLZ} and \textsc{LG}.

\section{Future Work}
 Following our work, several interesting problems deserve to be studied in future. For example, can the coreset construction time of Algorithm~\ref{alg-coreset} be improved, like the fast net construction method proposed by~\cite{har2006fast} in doubling metrics? 
 In theory, it is interesting to study other optimization problems involving outliers by using greedy strategy. 
Also, if we replace $k$-center clustering by $k$-center clustering with outliers, it may be possible to improve the robustness for the applications in  deep learning~\citep{DBLP:conf/iclr/ColemanYMMBLLZ20}, active learning~\citep{DBLP:conf/iclr/SenerS18}, and fairness~\citep{DBLP:conf/icml/KleindessnerAM19}.

\appendix
\section{Proof of Claim \ref{pro-core}}
\label{sec-proof-c1}

Suppose $H$ is an $\alpha$-approximation of the instance (coreset) $S$. Let $H_{\mathtt{opt}}$ be the set of $k$ cluster centers yielding the optimal solution of $X$. Then we have 
\begin{eqnarray}
\phi_{0}(S, H)&\leq&\alpha \phi_{0}(S, H_{\mathtt{opt}});  \label{for-pro13}\\
\phi_0 (S,H)&\in& (1\pm\mu)\phi_{0}(X, H); \label{for-pro11}\\
\phi_0 (S,H_{\mathtt{opt}})&\in& (1\pm\mu)\phi_{0}(X, H_{\mathtt{opt}}). \label{for-pro12}
\end{eqnarray}
Combining the above inequalities, we directly have
\begin{eqnarray}
\phi_{0}(X, H)\leq \frac{1}{1-\mu}\phi_0 (S,H)\leq \frac{\alpha}{1-\mu}\phi_0 (S,H_{\mathtt{opt}})\leq \frac{\alpha (1+\mu)}{1-\mu}\phi_0 (X,H_{\mathtt{opt}}).
\end{eqnarray}
So $H$ is an $\frac{\alpha (1+\mu)}{1-\mu}$-approximation of $X$.

\section{Proof of Claim~\ref{cla-core}}
\label{sec-proof-cla-core}
We just need to prove the first inequality since the other one can be obtained by the same manner. 
Because each $B_j\subseteq \mathtt{Ball}(c_j, r_X)$ and each vertex $p$ is moved by a distance at most $\mu r_{\mathtt{opt}}$ based on (\ref{for-map}), we know that $f(B_j)\subseteq \mathtt{Ball}(c_j, r_X+\mu r_{\mathtt{opt}})$, i.e., $r'_E\leq r_X+\mu r_{\mathtt{opt}}$. 

Let $p_0$ be the vertex realizing $r_X=\phi_0(X, H)$, that is, there exists some $1\leq j_0\leq k$ such that $\mathtt{d}(c_{j_0},p_0)=r_X$.  The triangle inequality and (\ref{for-map}) together imply $\mathtt{d}(c_{j_0},f(p_0))\geq r_X-\mu r_{\mathtt{opt}}$. Hence $r'_E\geq r_X-\mu r_{\mathtt{opt}}$.

Overall, we have $|r'_E-r_X|\leq \mu r_{\mathtt{opt}}$.

\vskip 0.2in
\bibliography{randomized_gonzalez_jmlr}

\end{document}